\documentclass[11pt]{article}

\usepackage{amssymb,amsmath,graphicx,fullpage}
\usepackage{hyperref,url}
\usepackage{hyperref}
\usepackage{cleveref}
\usepackage{xcolor}
\usepackage{color}
\usepackage{pifont}
\usepackage{orcidlink}

\definecolor{BrickRed}{rgb}{0.8, 0.25, 0.33}

\newtheorem{theorem}{Theorem}
\newtheorem{proposition}{Proposition}
\newtheorem{definition}{Definition}

\newtheorem{corollary}{Corollary}

\newtheorem{example}{Example}
\newtheorem{lemma}{Lemma}
\newtheorem{remark}{Remark}
\newenvironment{proof}{\paragraph{Proof:}}{\hfill$\square$}
\usepackage{todonotes}

\def\Spect{\mathrm{Spect}}

\def\ND{\mathrm{ND}}
\def\ld{\mathrm{ld}}
\def\bld{\mathrm{bild}}

\def\GL{\mathrm{GL}}

\DeclareMathOperator{\arccosh}{arccosh}

\def\tX{{\tilde{X}}}
\def\calD{\mathcal{D}}

\def\tr{\mathrm{tr}}

\newcommand{\VV}{\mathcal{V}}
\newcommand{\WW}{\mathcal{W}}

\newcommand{\VPM}{\operatorname{VPM}^\circ}

\renewcommand{\max}{\operatorname{max}}
\renewcommand{\min}{\operatorname{min}}

\def\bVPM{{\overline{\operatorname{VPM}}}}

\renewcommand{\epsilon}{\varepsilon}

\newcommand{\Sym}{\operatorname{Sym}}

\newcommand{\PD}{\operatorname{PD}}

\renewcommand{\dh}{\operatorname{d_H}} 
\newcommand{\db}{\operatorname{d_B}} 
\def\AIRM{{\mathrm{AIRM}}}
\newcommand{\dairm}{\operatorname{d_{\AIRM}}}
\newcommand{\gairm}{\operatorname{g^{\text{AIRM}}}}
\newcommand{\dairmr}{\operatorname{d^{||}_{\text{AIRM}}}}
\newcommand{\dairmp}{\operatorname{d^{\rightarrow}_{\text{AIRM}}}}
\newcommand{\mairm}[1]{\|#1\|^{\text{AIRM}}}

\newcommand{\dpsi}{\operatorname{d_{\Psi}}}

\def\GL{\mathrm{GL}}

\def\diag{\mathrm{diag}}

\newcommand{\spread}{\mathbf{spread}}

\newcommand\no{\color{BrickRed} \ding{55}}
\def\inner#1#2{{\langle #1,#2\rangle}}

\def\Mat{\mathrm{Mat}}
\def\calC{{C}}

\usepackage{natbib}

\def\AIRM{{\mathrm{AIRM}}}

\def\st{{\ :\ }}
\def\bbR{\mathbb{R}}

\def\calC{\mathcal{C}}

\title{Geometric structures and deviations on James' symmetric positive-definite matrix bicone domain}

\author{Jacek Karwowski\thanks{Equal contribution.}~\orcidlink{0000-0002-8361-2912}\\
Department of Computer Science\\
University of Oxford, UK
\and
Frank Nielsen\footnotemark[1]~\orcidlink{0000-0001-5728-0726}\\
Sony Computer Science Laboratories Inc.\\
Tokyo, Japan
}

\date{}

\begin{document}

\maketitle

\begin{abstract}
Symmetric positive-definite (SPD) matrix datasets play a central role across numerous scientific disciplines, including signal processing, statistics, finance, computer vision, information theory, and machine learning among others. The set of SPD matrices forms a cone which can be viewed as a global coordinate chart of the underlying SPD manifold. Rich differential-geometric structures may be defined on the SPD cone manifold.
Among the most widely used geometric frameworks on this manifold are the affine-invariant Riemannian structure and the dual information-geometric log-determinant barrier structure, each associated with dissimilarity measures (distance and divergence, respectively). In this work, we introduce 
two new structures, a Finslerian structure and a dual information-geometric structure, both derived from James' bicone reparameterization of the SPD domain.
Those structures ensure that geodesics correspond to straight lines in appropriate coordinate systems. 
The closed bicone domain includes the spectraplex (the set of positive semi-definite diagonal matrices with unit trace) as an affine subspace, and the Hilbert VPM distance is proven to generalize the Hilbert simplex distance which found many applications in machine learning. 
Finally, we  discuss several applications of these Finsler/dual Hessian structures  and provide various inequalities between the new and traditional dissimilarities.
\end{abstract}


\section{Introduction}

Let  $\Mat(n)$ denote the vector space of $n\times n$ matrices with real coefficients, and let $\Sym(n)\subset \Mat(n)$ denote the vector subspace symmetric matrices.
We equip $\Mat(n)$ with the inner product $\inner{M_1}{M_2}=\tr(M_1^\top M_2)$ which induces the Frobenius norm $\|M\|_F=\sqrt{\inner{M}{M}}$ and the corresponding distance  $d_F(M_1,M_2)=\|M_1-M_2\|_F$, and consider the metric topology.
Let $\PD(n)\subset\Sym(n)$ denote the subset (cone) of symmetric positive-definite (SPD) matrices~\cite{bhatia2009positive} with positive eigenvalues, and $\ND(n)=-\PD(n)$, the symmetric negative-definite cone. 
SPD matrix datasets are omnipresent in many scientific areas ranging from statistics and probability (Gaussian distributions) to signal processing (Kalman filtering), medical imaging~\cite{pennec2019riemannian} (diffusion tensor imaging), optimization~\cite{cherian2016riemannian}, and computer vision~\cite{sra2016positive,yair2019parallel} just to name a few.

This paper aims to compare two usual dissimilarities on the SPD cone frequently used in applications:
Namely, the affine-invariant Riemannian metric distance described in~\S\ref{sec:AIRM} (AIRM distance for short) and the logdet divergence~\cite{cichocki2015log} detailed in \S\ref{sec:logdet}) with two novel dissimilarities based on the bicone representation~\cite{james1973variance} of the SPD cone:

\begin{itemize}
\item The Hilbert SPD bicone distance~\cite{karwowski2025hilbertgeometrysymmetricpositivedefinite}, and 
\item The bicone logdet divergence. 
\end{itemize}

We first recall the coupling of the AIRM structure with the information-geometric logdet Hessian structure on the SPD cone in \S\ref{sec:AIRM} and \S\ref{sec:logdet}, and then describe and study properties of the new Finsler and dual bicone logdet Hessian structures underlying the Hilbert and the logdet  dissimilarities 
in \S\ref{sec:FinslerHilbert} and \S\ref{sec:vpmlogdet}, respectively.
Then we compare the AIRM/logdet dissimilarities with the new Hilbert/bicone logdet  dissimilarities in \S\ref{sec:compare}.
Finally, we summarize the results and discuss perspectives of this work in \S\ref{sec:concl}.

A summary of our notations is given in Table~\ref{tab:placeholder} of the Appendix~\ref{sec:notations}.

\subsection{Riemannian geometry: Affine-invariant Riemannian metric}\label{sec:AIRM}

From the viewpoint of differential geometry, the domain $\PD(n)$ is interpreted as the global chart of the SPD cone manifold $\calC$ of dimension $\frac{n(n+1)}{2}$ with tangent plane $T_p \calC$ at $p\in \calC$ identified with $\Sym(n)$. Among the many families of possible Riemannian metrics~\cite{thanwerdas2022riemannian} on $\calC$, the affine-invariant Riemannian metric~\cite{harandi2014manifold} (AIRM) is commonly used in applications since its 
Riemannian distance $d_\AIRM(X_1,X_2)$ and geodesics $\gamma_{X_1,X_2}^\AIRM(t)$ are available in closed-form:
\begin{eqnarray*}
\dairm(X_1,X_2) &=& \sqrt{\sum_{i=1}^n \log^2 \lambda_i(X_2^{-1}X_1)},\\
\gamma_{X_1,X_2}^\AIRM(t) &=& X_1^{\frac{1}{2}}\, \left(X_1^{-\frac{1}{2}}\, X_2\, X_1^{-\frac{1}{2}}\right)^t \, X_1^{\frac{1}{2}} ,
\end{eqnarray*}
 and satisfy invariance properties to congruence and matrix inversion:
\begin{eqnarray*}
\dairm(X_1,X_2) &=& \dairm(A X_1 A^\top, A X_2 A^\top), \forall A\in\GL(n)\\
\dairm(X_1,X_2) &=&  \dairm(X_1^{-1},X_2^{^1}),
\end{eqnarray*}
where $\GL(n)$ denotes the General Linear group of $\Mat(n)$.

Moreover, the AIRM  metric $g_P^\AIRM$ coincides up to a scaling factor with the Fisher information metric  of the centered Gaussian family~\cite{james1973variance,skovgaard1984riemannian}:
$$
g_X^\AIRM(S_1,S_2) = \tr\left(X^{-1}S_1 X^{-1}S_2\right), \forall X\in\PD(n), \forall S_1,S_2\in\Sym(n). 
$$

The AIRM is thus also called the trace metric~\cite{cherian2016riemannian} and the corresponding Riemannian distance was historically first calculated in the more general setting of the Siegel upper space~\cite{siegel1943symplectic} of complex matrices with imaginary positive-definite parts.

\subsection{Dual information geometry: Hessian metric}\label{sec:logdet}

In information geometry~\cite{IG-2016,shima2007geometry}, any potential function $\phi$ on a $n$-dimensional affine manifold $(\calC,\nabla)$ equipped with a torsion-free flat connection $\nabla$ defines a Riemannian metric $g=\nabla d\phi$ where $d$ denotes the exterior derivative.
A coordinate system $\theta(\cdot)$ yields vector fields $\left\{\partial_1=\frac{\partial}{\partial\theta_1}, \ldots, \partial_n=\frac{\partial}{\partial\theta_n} \right\}$ on $\calC$.
A coordinate system is said $\nabla$-affine~\cite{shima2007geometry} when the Christoffel symbols $\Gamma_{ij}^k(\theta)$ of the connection $\nabla$ vanish.
An affine transformation of $\theta$, $\bar\theta=A\theta+b$ for $A\in\GL(n)$ and $b\in\bbR^n$ yields another $\nabla$-affine coordinate system.
The metric $g$ can be expressed in local coordinates $\theta$ as $G(\theta)=[G_{ij}(\theta)]$ with $G_{ij}(\theta)=\frac{\partial^2 \Psi(\theta)}{\partial \theta^i\partial \theta^j}$ (written compactly as $G_{ij}=\partial_i\partial_j\Psi$) where $\phi(p)=\Psi(\theta(p))$ (i.e. $\Psi$ is the potential function expressed in the $\theta$-coordinate system). The structure $(\nabla,g)$ is a called a Hessian structure~\cite{shima2007geometry} on $\calC$, and yields another dual Hessian structure $(\nabla^*,g)$ where $\nabla^*$ is such that $\frac{\nabla+\nabla^*}{2}=\nabla^g$ where $\nabla^g$ is the Levi-Civita metric connection. The dual connection $\nabla^*$ coupled to the metric $g$  that is induced by the Legendre dual potential function $\phi^*$, i.e., $g=\nabla^* d\phi^*$ expressed in the local coordinates $\eta$ as
 $G^{ij}(\eta)=\frac{\partial^2 \Psi^*(\eta)}{\partial \eta^i\partial \eta^j}$ where $\phi^*(p)=\Psi^*(\eta(p))$. 
Let $\{\partial^1=\frac{\partial}{\partial\eta_l}, \ldots, \partial^n=\frac{\partial}{\partial\eta_n} \}$ be the vector fields induced by the $\nabla^*$-affine coordinate system $\eta$. (Once the $\nabla$-affine coordinate system is chosen up to an affine transformation, the dual $\nabla$-affine coordinate system is unique.)
Then $G^{ij}(\eta)=\partial^i\partial^j \Psi^*(\eta)$, and the dual coordinate systems are mutually orthogonal: $G(\theta)G^(\eta)=I_{n,n}$, the identity matrix. 

The dual potential functions $\phi$ and $\phi^*$ combine into a canonical divergence 
$\calD(p,q)=\phi(p)+\phi^*(q)-\sum_{i=1}^n \theta^i(p)\eta_i(q)$ 
 which can be expressed equivalently as dual Bregman divergences~\cite{Bregman-1967} in the dual $\nabla$-affine and $\nabla^*$-affine coordinate systems, $\theta(\cdot)$ and $\eta(\cdot)$, respectively.
This dually flat information geometry of the SPD cone manifold $\calC$ was studied in~\cite{ohara1996dualistic} where $\theta(p(X))=X$ and $\eta(p(X))=-X^{-1}$.
The dual potential functions on $\calC$ are $\phi(p)=-\log\det P(\theta(p))$ and $\phi^*(p)=-\log\det P(\eta)-n$ and the corresponding dual Bregman divergences
are $B_\Psi(P_1:P_2)=\log\det(P_2P_1^{-1})+\tr(P_2^{-1}P_1)-n$ obtained for $\Psi(\theta)=-\log\det\theta$ (commonly called the log-det divergence~\cite{cichocki2015log} or sometimes the Burg matrix divergence~\cite{davis2006differential} or Itakura-Saito matrix divergence~\cite{crammer2009gaussian}) and 
$B_{\Psi^*}(P_2^{-1}:P_1^{-1})=B_F(P_1:P_2)$ obtained for $\Psi^*(\eta)=-\log\det\eta-n$.
The Hessian potentials $\nabla^2 \Psi(\theta)$ and $\nabla^2 \Psi^*(\eta)$ induces a Riemannian metric $g^\Psi$ which coincides with the AIRM trace metric.
Although the logdet (ld) divergence is often used in information geometry, other dual potential functions~\cite{yoshizawa1999dual,ohara2014geometry,amari2014information} on $\calC$ have been considered in the literature.

\subsection{James' bicone}\label{sec:bicone}

James~\cite{james1973variance} considered the following two diffeomorphic mappings of $\PD(n)$ onto the VPM bicone: 
$$
\VPM(n)=\{X\in\PD(n) \st 0\prec X\prec I \},
$$ where $I$ denotes the identity matrix:
\begin{eqnarray*}
V(X) &=& X (I+X)^{-1}, \quad V^{-1}(X)=X (I-X)^{-1}\\
P(X) &=& (I+X)^{-1}, \quad P^{-1}(X)=X^{-1}-I.
\end{eqnarray*}

We have the eigenvalues $\lambda_i(V(X))\in(0,1)$ and $\lambda_i(P(X))\in(0,1)$, i.e., both $V(X)\in\VPM(n)$ and $P(X)\in\VPM(n)$.

\begin{proposition}
The eigenvalues $\lambda_i(V(X))\in(0,1)$ and $\lambda_i(P(X))\in(0,1)$, i.e., both $V(X)\in\VPM(n)$ and $P(X)\in\VPM(n)$.
\end{proposition}

\begin{proof}
Let $X=U\, D\, U^\top$ denote the eigen decomposition of SPD matrix $X$ with $U$ an orthonormal matrix of $O(n)$ such that $U\, U^\top=I$ 
and $D=\diag((\lambda_i)_i)$ a diagonal matrix with eigenvalues $\lambda_i=\lambda_i(X)>0$.
Since we can rewrite the matrix identity as $I=U\, \diag((1)_i)\, U^\top$ for any $U\in O(n)$, we have $I+X=U \diag((1+\lambda_i)_i) U^\top$ and it follows that we get 
$P(X)=(I+X)^{-1}=U^{-\top}\, \diag((\frac{1}{1+\lambda_i})_i)\, U^{-1}$.
Thus we can rewrite transformation $P(X)$ as $(I+X)^{-1}= V\, \diag((\frac{1}{1+\lambda_i})_i)\, V^\top $ where  $V=R^{-\top}\in O(n)$ is another orthonormal matrix. 
Thus the eigenvalues of $P(X)=(I+X)^{-1}$ are $\frac{1}{1+\lambda_i}\in (0,1)$ since $\lambda_i>0$.
Now, similarly $V(X)=O\, \diag((\frac{\lambda_i}{1+\lambda_i})_i) O^{-1}$ for $O\in O(n)$.
Therefore its eigenvalues fall in the range $(0,1)$.
\end{proof}

Notice that if $J\in\VPM(n)$ then $J^{-1}\not\in\VPM(n)$.
The mapping $V(X)$ may be thought as the mapping of covariance matrices onto the VPM while the mapping $P(X)$ can be interpreted as the mapping of 
 precision matrices onto the VPM. This allows to consider extended Gaussians with potential degenerate covariance and/or precision matrices~\cite{james1973variance,stein2023category}.

It is worth to note that the VPM domain (and its closure $\bVPM(n)=\{X\in\PD(n) \st 0\preceq X\preceq I \}$) is considered in quantum information theory~\cite{nielsen2010quantum} and effect algebra~\cite{geher2020coexistency}. 
Indeed, the domain $\bVPM(n)$ models the matrix effects~\cite{montiel2025foundations} in   positive operator-valued measures (POVMs).

The VPM domain is also important for studying algebraic Ricatti~\cite{lancaster1995algebraic, jedra2022minimal} where the mapping $V: \PD(n)\rightarrow \VPM(n), X\mapsto V(X)=X (I+X)^{-1}$ introduces a normalized matrix transform ensuring eigenvalues in $(0,1)$.

\begin{figure}
\centering
\includegraphics[width=0.48\textwidth]{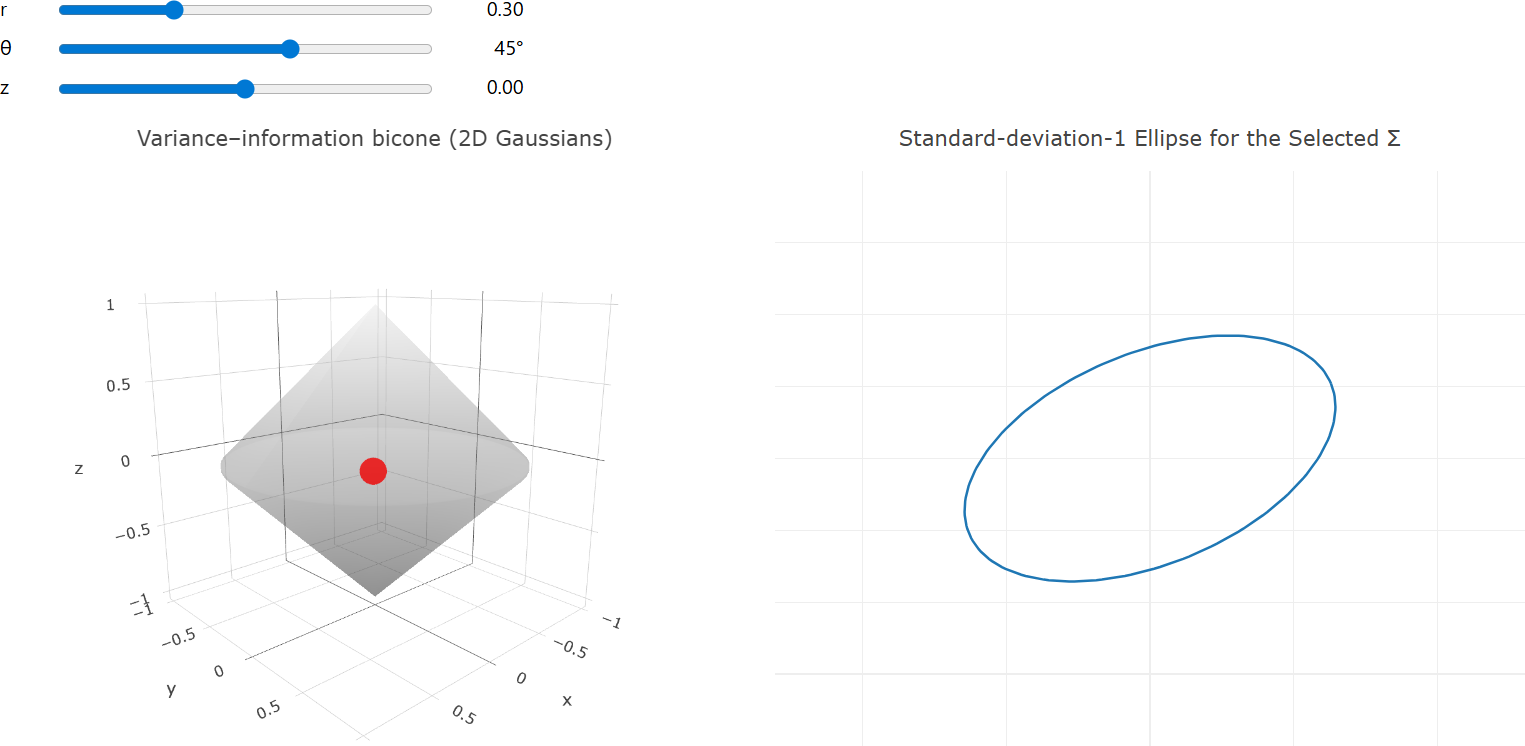}
\includegraphics[width=0.48\textwidth]{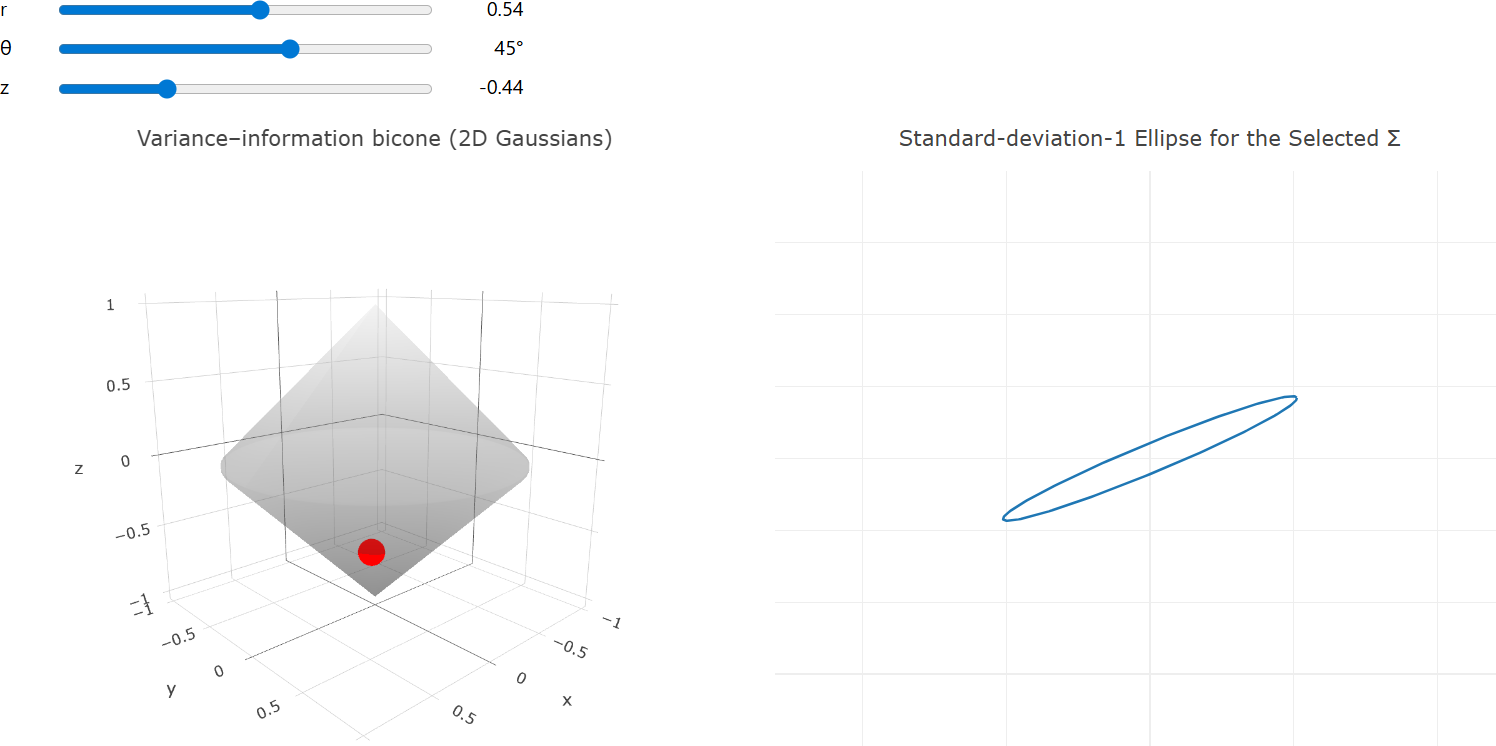}
\caption{Screenshots of James' 3D bicone model with corresponding bivariance centered Gaussians.}\label{fig:James3dbicone}
\end{figure}

\begin{definition}[James' map]\label{def:iota}
    We let $\iota: \PD(n) \to \VPM(n)$ denote the diffeomorphism $\iota(P) = P(I + P)^{-1}$ between the space of symmetric positive-definite matrices $\PD(n)$ of dimension $n$ and $\VPM(n)$. Its inverse is then $\iota^{-1}(X) = X(I - X)^{-1}$.
\end{definition}

\begin{proposition}[Differential of $\iota$]\label{prop:diota}
    For any $P \in \PD(n)$ and $V \in T_P\PD(n) \simeq \Sym(n)$, we have:
    \[
        d\iota_P(V) = (I + P)^{-1}V(I + P)^{-1}
    \]
\end{proposition}

\begin{lemma}[Inverse formulas for $\iota$]\label{lemma:woodbury}
    Let $P \in \PD(n)$. Then:
    \begin{enumerate}
        \item $\iota(P) = (I + P^{-1})^{-1}$, and
        \item $(I - \iota(P))^{-1} = I + P$.
    \end{enumerate}
\end{lemma}
\begin{proof}
    For the first part, we have:
    \[
    (I + P)P^{-1} = P^{-1} + I = I + P^{-1} \implies P(I + P)^{-1} = (I + P^{-1})^{-1}
    \]
    For the second part, from the Woodbury formula for any $X \in \VPM(n)$ we have:
    \[
    (I - X)^{-1} = I + X(I - X)^{-1} = I + \iota^{-1}(X)
    \]
    Thus:
    \[
        (I - \iota(P))^{-1} = I + \iota^{-1}(\iota(P)) = I + P
    \]
\end{proof}

When $n=2$, the VPM can be visualized as a 3D Lorentz bicone (Figure~\ref{fig:VPMLorentz}).
In particular, James~\cite{james1973variance} used the following cylindrical polar parameterization of the 3D bicone:
\begin{eqnarray*}
J=\left[\begin{array}{ll}
a & c\\
c & b
\end{array}
\right] &\mapsto&
\left(
r=\frac{\lambda_1-\lambda_2}{1+\lambda_1+\lambda_2+\lambda_1\lambda_2},
\theta=\arctan \frac{2c}{a-b},
z=\frac{\lambda_1\lambda_2-1}{1+\lambda_1+\lambda_2+\lambda_1\lambda_2}
\right)\\
&\mapsto&  (r\cos\theta,r\sin\theta,z).
\end{eqnarray*}

See Figure~\ref{fig:James3dbicone} for some screenshots of an interactive application.

The two eigenvalues $\lambda_1$ and  $\lambda_2$ of $J$ are
$$
\lambda_1=\frac{a+b-\sqrt{(a-b)^2+4c^2}}{2}, \quad \lambda_2=\frac{a+b+\sqrt{(a-b)^2+4c^2}}{2}.
$$

\begin{figure}
\centering
\includegraphics[width=0.45\textwidth]{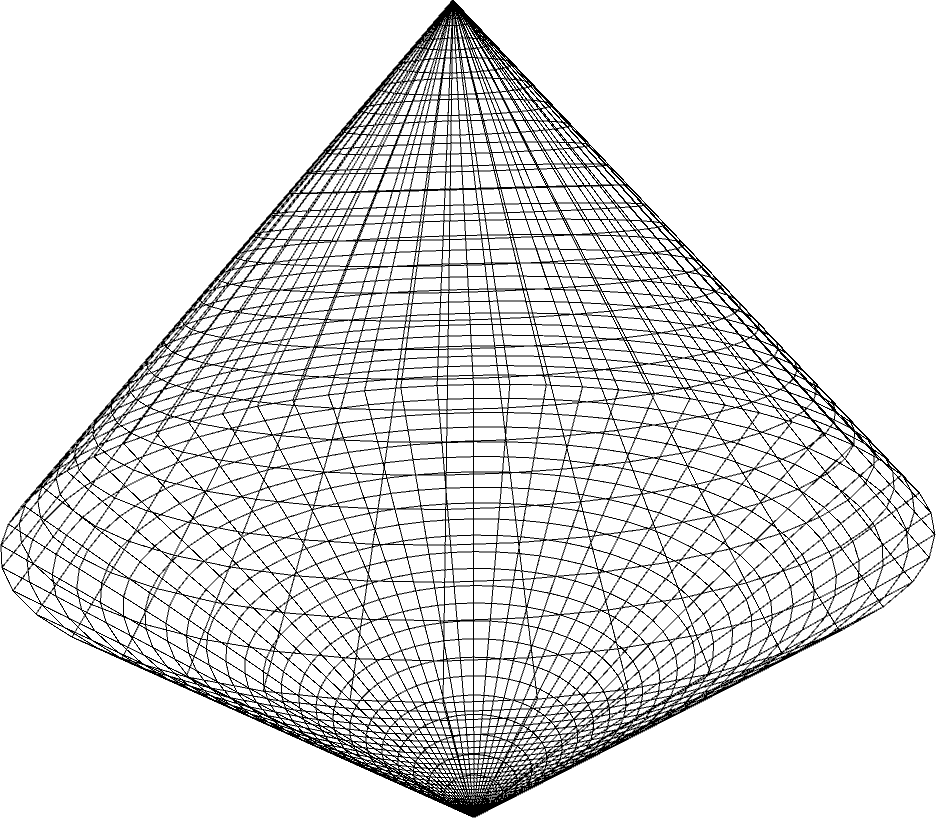}
\includegraphics[width=0.5\textwidth]{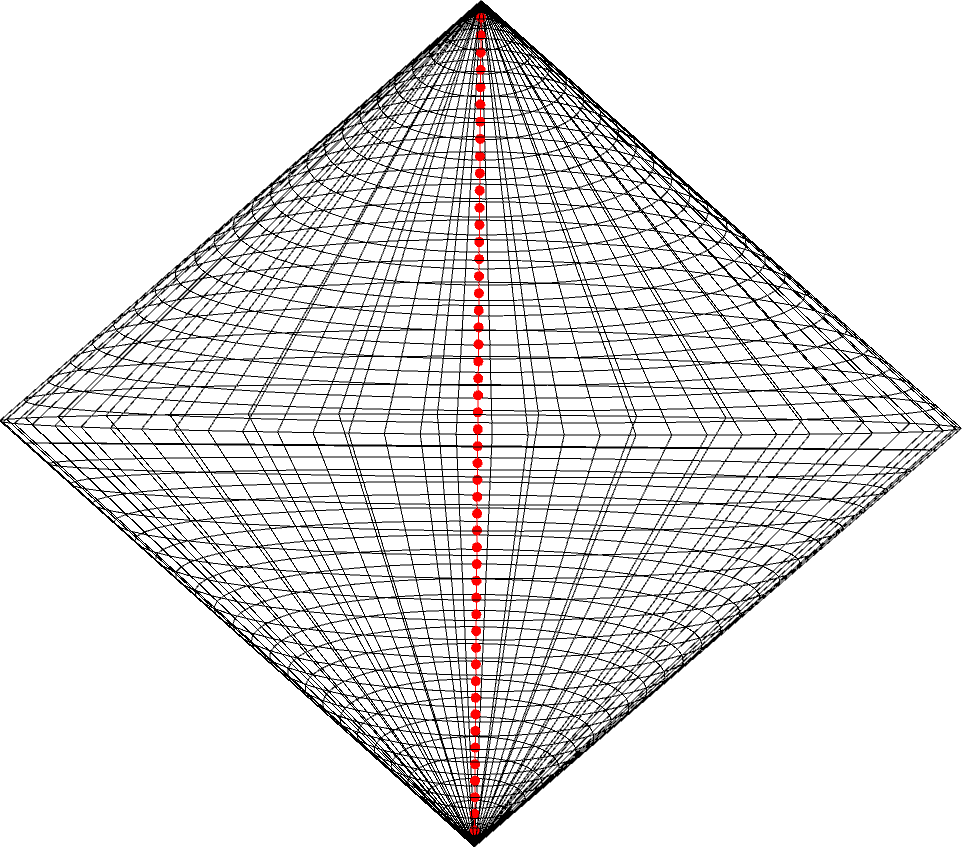}

\caption{$\VPM(2)$ can be visualized as a 3D Lorentz bicone, here shown on the left in a slanted view for better perception, 
and with the open pregeodesic joining $0$ to $I$, i.e. $\{0 \prec \alpha I\prec I  \ :\ \alpha\in(0,1)\}$.
}\label{fig:VPMLorentz}
\end{figure}

The closure of the $\VPM(n)$ allows to consider the extended Gaussian distributions with possible singular covariance and/or precision matrices~\cite{stein2023category}.
In particular, Dirac distributions are considered as extended Gaussians with singular covariance matrices and affine subspaces are considered as extended Gaussians with singular precision matrices. 
Handling extended Gaussians allow to model uncertainty with non-determinism induced by relations between variables, a framework pioneered by Willems for modeling open stochastic systems~\cite{willems2012open}.

The log cross-ratio Hilbert distance on the bounded convex set $\VPM(n)$ was studied in~\cite{karwowski2025hilbertgeometrysymmetricpositivedefinite}. 
Hilbert distance on the SPD cone finds numerous applications in information theory~\cite{reeb2011hilbert} and statistics~\cite{chen2021stochastic} among others.
The Riemannian cone manifold $\calC$ admit many potential global charts:
Information geometry induced by the dual logdet potentials $\Psi_\ld(X)=-\log\det(X)$ and $\Psi_\ld^*$ yield two global charts $\Theta=\PD(n)$ and $H=-\PD(n)=\ND(n)$ (negative-definite cone), and the differential geometry of the Hilbert distance defined on the SPD bicone $\VPM(n)$ induces a Finsler structure~\cite{shen2001lectures} and distance on $\calC$.
Furthermore, by considering the bilogdet function $\Psi_\bld(X)=-\log\det(X)-\log\det(I-X)$ on the VPM, we obtain the space of symmetric matrices $\Sym(n)$ as the gradient space. 
Observe that that the following identity holds: $\log\det(X)=\tr\log(X)$.
Thus computing the logdet potential function requires $O(n^\omega)$ operations where $\omega$ denote the exponent in the complexity of matrix inversion, multiplication, and eigendecomposition in $\Mat(n)$ ($\omega\approx 2.37$).

\begin{figure}
\includegraphics[width=\textwidth]{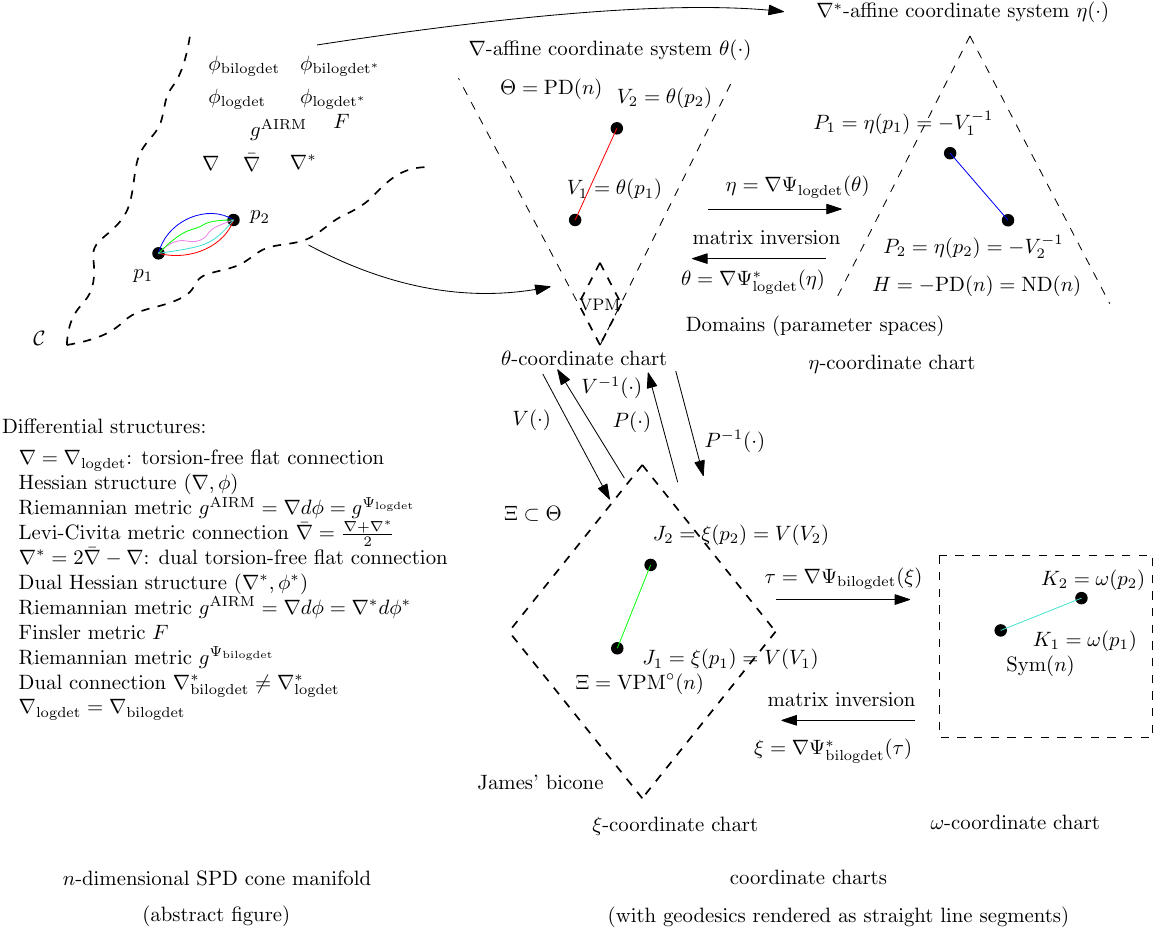}

\caption{Both the dual SPD cones $\Theta=\PD(n)$ and $H=\ND(n)=-\PD(n)$ via the logdet potential function, and the bicone $\Xi=\VPM(n)$ and its dual $\Sym(n)$ are global coordinate charts of the SPD cone manifold $\calC$.}\label{fig:coordcharts}
\end{figure}

\subsection{Contributions and paper outline}

We concisely summarize our main contributions as follows:

\begin{itemize}

\item We prove that the Hilbert VPM distance generalizes the Hilbert simplex distance in Theorem~\ref{thm:vpmhs},

\item We report the closed-form expression for the constant-speed parameterization of the Hilbert geodesics in the standard simplex in Theorem~\ref{sec:hsgeodesic},

\item By analogy, we give  the closed-form expression for the constant-speed parameterization of the Hilbert geodesics in VPM domain in Theorem~ \ref{thm:vpmgeodesic},

\item We report a tight lower bound on the Hilbert distance with the affine-invariant Riemannian distance restricted in the VPM domain in Theorem~\ref{thm:lower-dh-dairmr},

\item Similarly, we provide a tight upper bound on the Hilbert distance with the pushed-forward affine-invariant Riemannian distance in Theorem~\ref{thm:ubairmp},

\end{itemize}

The paper is organized as follows:
We first describe the Finsler structure induced by the Hilbert geometry of the VPM distance (formerly introduced in our prior work~\cite{karwowski2025hilbertgeometrysymmetricpositivedefinite}) in Section~\ref{sec:FinslerHilbert}.
Then we proceed by introducing the bilogdet barrier function on the VPM and study its induced information geometry in Section~\ref{sec:vpmlogdet}.
In Section~\ref{sec:compare}, we compare the VPM Hilbert distance with the widely used AIRM distance by considering either its restriction of the VPM domain or its pushed-forward version.
Finally, we summarize the key results and discuss about potential applications in quantum information theory and control theory in~\ref{sec:concl}.
In Appendix, we provide a table of notations in \S\ref{sec:notations}, and some symbolic computing code snippets in \S\ref{sec:maxima}.

\section{Hilbert VPM distance and Finsler structure}\label{sec:FinslerHilbert}

\subsection{Preliminaries}

\begin{definition}[Matrix norms]
    For $V \in \Sym(n)$, with real eigenvalues $\lambda_1(V), \ldots, \lambda_n(V)$, we define:
    \begin{enumerate}
        \item Frobenius norm: $\|V\|_F = \sqrt{\sum_{i=1}^n \lambda^2_i(V)}$
        \item Operator norm ($2$-norm): $\|V\|_{op} = \max_{1 \leq i \leq n} |\lambda_i(V)|$
    \end{enumerate}
\end{definition}

\begin{lemma}[Frobenius-operator norm inequality]\label{lemma:matrix-norms-inequalities}
    For any $V \in \Sym(n)$, we have inequalities:
    \[
        \|V\|_{op}\leq \|V\|_F \leq \sqrt{n}\, \|V\|_{op}
    \]
\end{lemma}

\begin{lemma}[Range-$l_2$ inequality]\label{lemma:range-l2}
For any $x\in\mathbb{R}^n$, we have
\(
\max_{1 \leq k \leq n} x_k - \min_{1 \leq k \leq n} x_k \le \sqrt{2} \|x\|_2
\).
\end{lemma}
\begin{proof}
Let $i = \arg\max_{k} \{x_k\}$ and $j = \arg\min_{k} \{x_k\}$. Then:
\[
\max_{1 \leq k \leq n} x_k-\min_{1 \leq k \leq n} x_k = x_i - x_j = \langle x, e_i-e_j\rangle \le \|x\|_2 \|e_i-e_j\|_2\leq\sqrt{2}\|x\|_2.
\]
where the inequality follows from Cauchy-Schwarz.
\end{proof}

Hilbert geometry have been studied for various bounded convex domains~\cite{nielsen2018clustering,georgiou2015positive}.
The Hilbert geometry of the SPD domain with its induced Finsler structure was studied in~\cite{mostajeran2024differential}.

The Hilbert distance for two parameters $J_1$ and $J_2$ on the VPM has been calculated in~\cite{karwowski2025hilbertgeometrysymmetricpositivedefinite}:

\begin{theorem}[{\cite[Theorem 14]{karwowski2025hilbertgeometrysymmetricpositivedefinite}}]
    Given two matrices $J_1, J_2 \in \VPM(n)$: 
    \[
    \dh(J_1, J_2) = \log \frac{\max\left(\lambda_{\max}(J_2^{-1}J_1), \lambda_{\max}\!\big((I - J_2)^{-1}(I - J_1)\big)\right)}{\min\left(\lambda_{\min}(J_2^{-1}J_1),  \lambda_{\min}\!\big((I - J_2)^{-1}(I - J_1)\big)\right)}
    \]
    where $\lambda_{\min}(V)$ and $\lambda_{\max}(V)$ denote the minimal and maximal real eigenvalues of symmetric matrix $V$.
\end{theorem}

Since for any positive $a,b>0$, we have $\max\{a,b\}\geq \frac{a+b}{2}$ and $\min\{a,b\}\leq \frac{a+b}{2}$, we get the following lower bound on the VPM Hilbert distance:

\begin{corollary}[Lower bound]
The Hilbert VPM distance is lower bounded as follows:
$$
\dh(J_1, J_2) \geq \log \frac{\lambda_{\max}(J_2^{-1}J_1) + \lambda_{\max}\!\big((I - J_2)^{-1}(I - J_1)}{\lambda_{\min}(J_2^{-1}J_1) +  \lambda_{\min}\!\big((I - J_2)^{-1}(I - J_1)\big)}.
$$
\end{corollary}

Note that the matrix $J_2^{-1}J_1$ can have eigenvalues with multiplicities:
For example consider $J_1=\alpha I$ and $H_2=\beta I$ for $\alpha\in (0,1)$ and $\beta\in (0,1)$.
Then we have $J_2^{-1}J_1=\frac{\alpha}{\beta}I$ with one distinct eigenvalue only, and 
$(I-J_2)^{-1}(I-J_1)=\frac{1-\alpha}{1-\beta}I$ with one distinct eigenvalue only.
In that case, we get $ \dh(J_1, J_2)=\log\frac{\max\left\{\frac{\alpha}{\beta},\frac{1-\alpha}{1-\beta}\right\}}{\min\left\{\frac{\alpha}{\beta},\frac{1-\alpha}{1-\beta}\right\}}$. We check that $\dh(J_1, J_2)=0$ if and only if $\alpha=\beta$.

\begin{example}
Consider the following two matrices
$$
J_1 =\left[
\begin{array}{cc}
\frac{7}{20} & -\frac{3\sqrt{3}}{20} \\
-\frac{3\sqrt{3}}{20} & \frac{13}{20}
\end{array} \right], \quad
J_2 =\left[
\begin{array}{cc}
\frac{11}{20} & -\frac{\sqrt{3}}{20} \\
-\frac{\sqrt{3}}{20} & \frac{9}{20}
\end{array}\right].
$$
We have $\Lambda(J_1)=\{\frac{1}{5},\frac{4}{5}\}$ and $\Lambda(J_2)=\{\frac{2}{5},\frac{3}{5}\}$.
Using the symbolic calculation reported in Appendix~\ref{sec:maxima}, we find
$$
\dh(J_1, J_2) = \log\frac{47+\sqrt{673}}{47-\sqrt{673}} \approx 1.242398973577776.
$$
\end{example}

Notice that the spectraplex $\Spect_n=\{ X\succeq 0 \ :\ \tr(X)=1\}\subset\bVPM$ is an affine subspace of the closed VPM. The spectraplex is the semi-definite counterpart of the standard simplex. Let $\Spect_n^\circ=\{ X\in\PD(n) \ :\ \tr(X)=1\}\subset\VPM$ be the open positive-definite spectraplex.
The open spectraplex corresponds to embedding the open standard simplex (probability simplex) $\Delta_n^\circ$ into diagonal matrices:
$p\in\Delta_n^\circ\mapsto D_p:=\diag(p_1,\ldots,p_n)\in\Spect_n$.
The Hilbert distance on the simplex domain~\cite{nielsen2018clustering,nielsen2023non} is
$$
d_H^{\Delta_n}(p,q) = \log\frac{\max_i \frac{p_i}{q_i}}{\min_i \frac{p_i}{q_i}}.
$$
This Hilbert simplex distance enjoys a contraction property under linear transforms which makes it useful in analyzing Sinkhorn-type algorithms in optimal transport~\cite{holliday2005shannon, chizat2018scaling}.

We have
$$
\dh(D_p,D_q) = \log \frac{\max\{\max_{i=1}^n \frac{p_i}{q_i} ,\max_{i=1}^n \frac{1-p_i}{1-q_i} \}}{\min\{\min_{i=1}^n \frac{p_i}{q_i} ,\min_{i=1}^n \frac{1-p_i}{1-q_i} \}}.
$$

\begin{theorem}\label{thm:vpmhs}
The Hilbert VPM distance generalizes the Hilbert simplex distance:
$$
\dh(D_p,D_q) = d_H^{\Delta_n}(p,q), \quad\forall p,q\in\Delta_n^\circ.
$$
\end{theorem}

\begin{proof}
Define $a_i=\frac{p_i}{q_i}$ and $b_i=\frac{1-p_i}{1-q_i}$.
Since $0<q_i<1$, consider $b_i-1=\frac{1-p_i}{1-q_i}-1=\frac{q_i-p_i}{1-q_i}$.

Because $1-q_i>0$, the sign of $b_i-1$ equals the sign of $q_i-p_i$,
which is the opposite of the sign of $p_i-q_i$. But $a_i-1=\frac{p_i-q_i}{q_i}$.
Hence, we have $(a_i-1)(b_i-1)\le 0$.
Therefore, for each $i$, the numbers $a_i$ and $b_i$ lie on opposite
sides of $1$.

Let $M=\max_i a_i$ and $m=\min_i a_i$.
If $M>1$, then for the corresponding index $k$ we have $b_k<1$,
so no $b_i$ can exceed $M$. If instead $M\le 1$, then all $b_i\ge 1$,
and the same symmetry argument applies. Thus we get
$$
\max\{\max_i a_i,\max_i b_i\}=M.
$$

Similarly, we have
$$
\min\{\min_i a_i,\min_i b_i\}=m.
$$

Therefore, we get
$$
\frac{\max\{\max_i a_i,\max_i b_i\}}
{\min\{\min_i a_i,\min_i b_i\}}
=
\frac{M}{m}
=
\frac{\max_i \frac{p_i}{q_i}}
{\min_i \frac{p_i}{q_i}}.
$$
It follows that $\dh(D_p,D_q) = d_H^{\Delta_n}(p,q)$ since
$D_p D_q^{-1}=\diag(\frac{p_1}{q_1},\ldots,\frac{p_n}{q_n})$ 
and $(I-D_p) (I-D_q)^{-1}=\diag(\frac{1-p_1}{1-q_1},\ldots,\frac{1-p_n}{1-q_n})$.
\end{proof}

It is shown in~\cite{delaharpeHibertsMetricSimplices1993,nielsen2018clustering} that although straight lines are geodesics in the Hilbert geometry of the standard simplex, geodesics are not unique in general. Thus straight line segments are geodesics on the Hilbert VPM geometry but those geodesics are not unique. 
The fact that Hilbert geometry have line segment (pre)geodesics allows one to generalize and implement various algorithms like the smallest enclosing ball of the set of points~\cite{nock2005fitting}, etc. 
The pregeodesic linking $J_1$ to $J_2$ is $\tilde\gamma_{J_1J_2}(t)=(1-t)J_1+t J_2$.
To obtain the geodesic we need to reparameterize the pregeodesic with unit speed:
$$
\gamma_{J_1J_2}(s)=\tilde\gamma_{J_1J_2}(s(t))
$$
such that $\dh(\gamma_{J_1J_2}(s),\gamma_{J_1J_2}(s'))=|s-s'|\, \dh(J_1,J_2)=|s(t)-s(t')|\, \dh(J_1,J_2)$.
To contrast with Hilbert geodesics, the AIRM geodesics are exponential arcs:
Indeed, $\gamma_{D_pD_q}^\AIRM(s)=\diag\left(q_1^s p_1^{1-s},\ldots, q_n^s p_n^{1-s}\right)$.

\begin{theorem}[Hilbert simplex geodesic parameterization]\label{sec:hsgeodesic}
The constant speed geodesic for the geodesic between $p$ and $q$ in the Hilbert simplex geometry is:
$$
\gamma_{pq}^{\Delta_n}(s)=(1-t(s))p+t(s)q,
$$
where
$$
t(s)=\frac{ 1-\left(\frac{M}{m}\right)^{1-s} }{ \left(\frac{M}{m}\right)^{1-s} \left(\frac{1}{M}-1\right)-\left(\frac{1}{m}-1)\right)}
$$
and $M=\max_{i=1}^n \frac{p_i}{q_i}$ and $m=\min\frac{p_i}{q_i}$.
\end{theorem}

\begin{theorem}[Hilbert VPM geodesic parameterization]\label{thm:vpmgeodesic}
The constant speed geodesic for the geodesic between $J_1$ and $J_2$ in the Hilbert simplex geometry is:
$$
\gamma_{pq}^{\VPM}(s)=(1-t(s))J_1+t(s)J_2,
$$
where
$$
t(s)=\frac{ 1-\left(\frac{M}{m}\right)^{1-s} }{ \left(\frac{M}{m}\right)^{1-s} \left(\frac{1}{M}-1\right)-\left(\frac{1}{m}-1)\right)}
$$  
and 
$$
M=\max\left(\lambda_{\max}(J_2^{-1}J_1), \lambda_{\max}\!\big((I - J_2)^{-1}(I - J_1)\big)\right)$$ 
and 
$$m=\min\left(\lambda_{\min}(J_2^{-1}J_1),  \lambda_{\min}\!\big((I - J_2)^{-1}(I - J_1)\big)\right).
$$
\end{theorem}


We may define the scaled VPM as $\VPM_\lambda(n)=\{X\in\PD(n) \st 0\prec X\prec \lambda I \}$ and the VPM cone
$K_{\mathrm{VPM}}=\{(\lambda, \VPM_\lambda(n)) \ :\ \lambda>0\}$. 
Then the Birkhoff projective metric on $K_{\mathrm{VPM}}$ coincides with $\dh$ for the subspace $(1,\VPM_1)$, see~\cite{lemmensbirkhoff}.
We may further embed SPD matrices on $\VPM_\lambda(n)$ by the mapping $V_\lambda(X) = \lambda X (I+X)^{-1}$ which guarantees eigenvalues to fall in the range $(0,\lambda)$.

Figure~\ref{fig:geocompare} displays the geodesic midpoints with respect to the AIRM structure (blue) and the Hilbert-Finsler structure (red) for two input $2\times 2$ SPD matrices (black). 

\begin{figure}
\centering
\begin{tabular}{cc}
\includegraphics[width=0.45\textwidth]{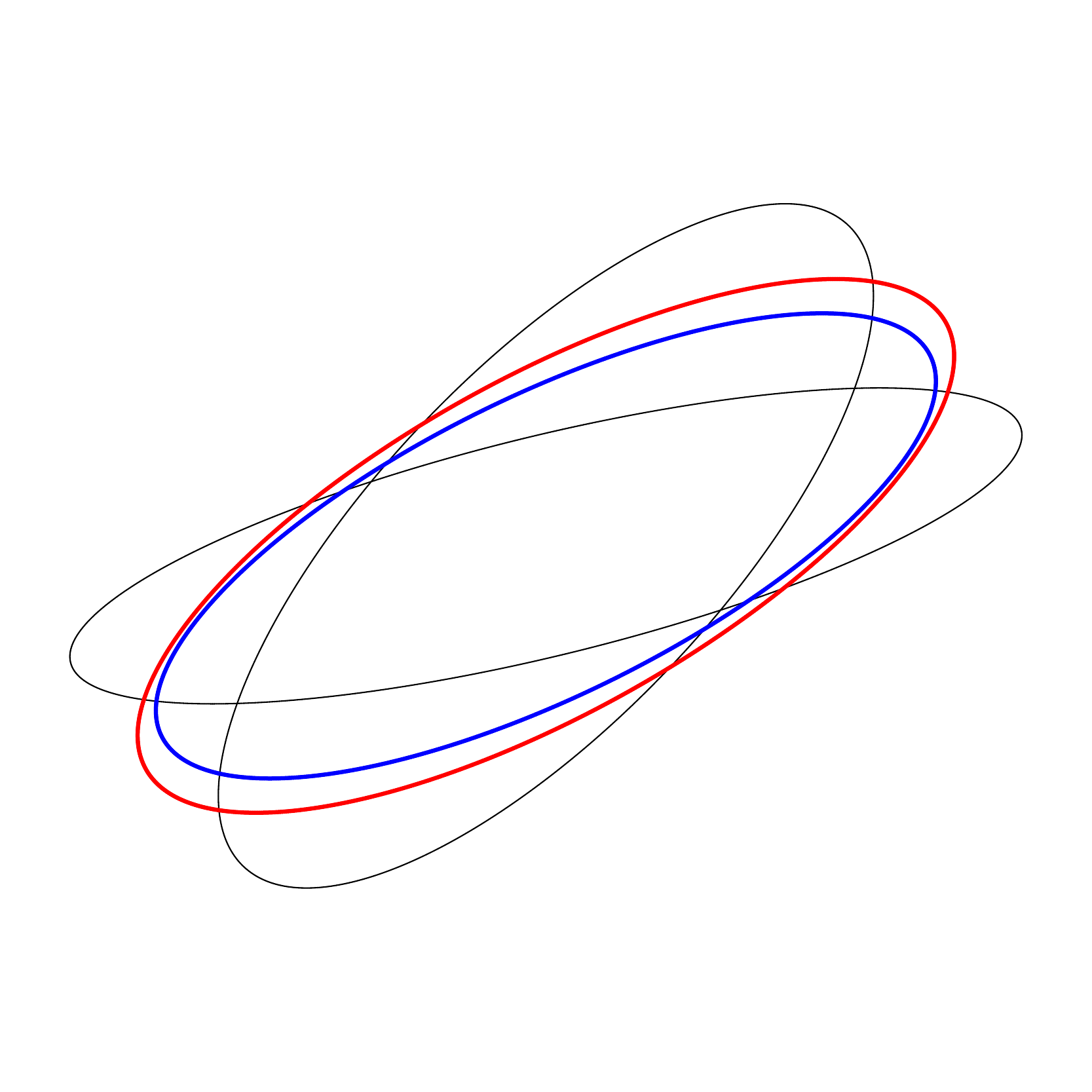} &
\includegraphics[width=0.45\textwidth]{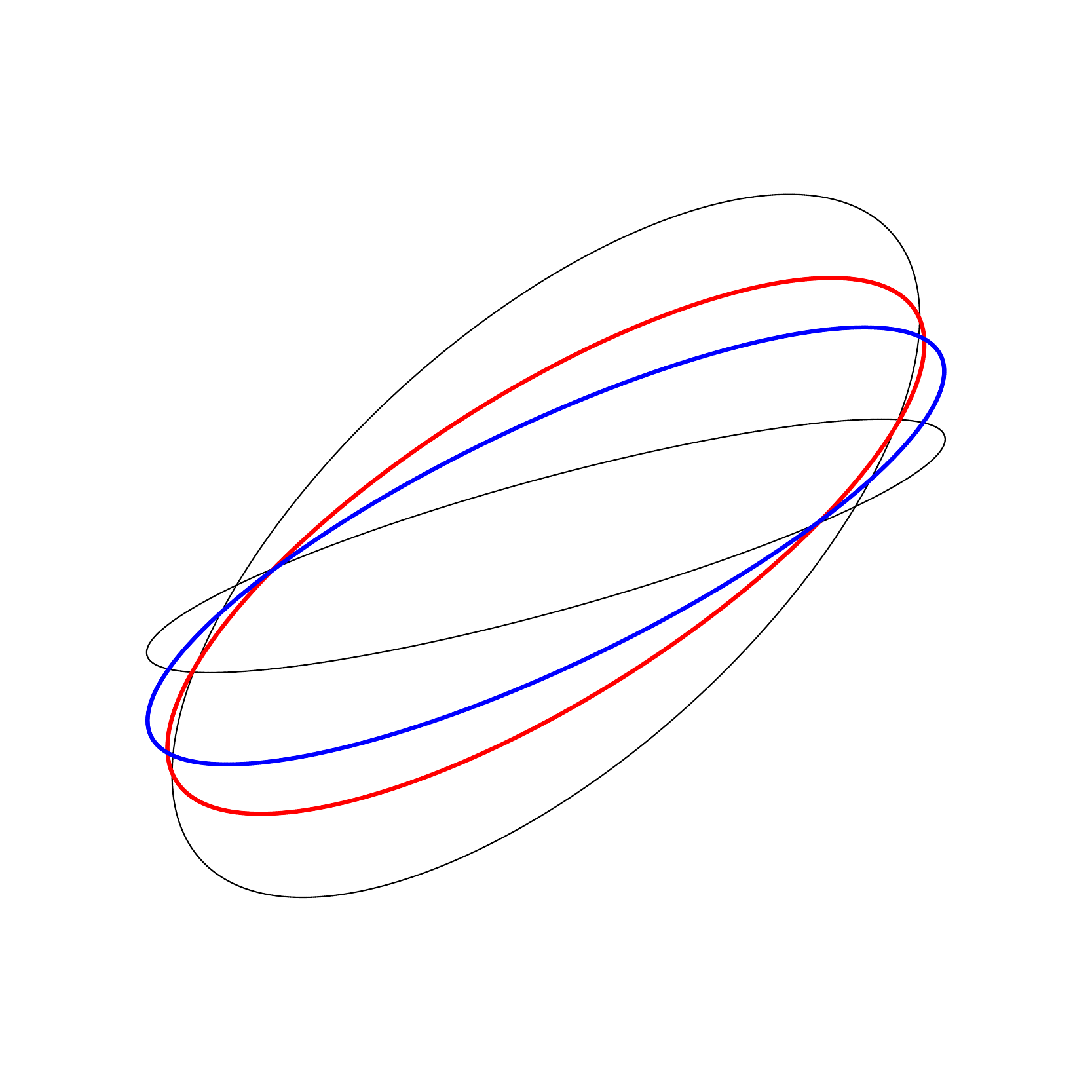} \\
\end{tabular}
\caption{AIRM (blue) and Hilbert-Finsler (red) geodesic midpoints of two SPD matrices (black).}\label{fig:geocompare}
\end{figure}

Many of our theorems below will make use of the following embedding of $\VPM(n)$ into the product cone $\PD(n) \times \PD(n)$.

\begin{definition}[Hat embedding]\label{def:hat}
    We define the hat embedding:
    \[
    (\hat{\phantom{x}}): \VPM(n) \to \VPM(n) \times \VPM(n) \subset \PD(n) \times \PD(n) \qquad  \text{ as } \qquad \hat{X} = (X, I - X).
    \]
\end{definition}

\subsection{Hilbert distance as matrix spread}

\begin{definition}
    For a symmetric matrix $S\in\Sym(n)$, we denote the matrix spread of $S$ by: \[
    \spread(S) = \lambda_{\max}(S) - \lambda_{\min}(S)\]
\end{definition}

To examine the relationship between $\spread$ and Hilbert distance, we need to borrow certain quantity $M$ which appears in the Birkhoff representation of the Hilbert distance.

\begin{definition}[Upper Birkhoff ratio $M$]\label{def:cone-gauge}
Let $\VV$ be a vector space and let $K\subseteq \VV$ be an open convex pointed cone inducing
an order $\preceq$. We define $M(\cdot,\cdot)$ for any $w \in K, v\in \VV$ by:
\[
M_K(v,w) = \inf\{\lambda>0:\ v\preceq \lambda w\}\ \in [0,\infty].
\]
\end{definition}

\begin{remark}[Birkhoff projective distance]
Given a vector space $\VV$ and an open convex pointed cone $K \subseteq \VV$ defining an order $\preceq$ on $\VV$,
we can define the Birkhoff distance $\db$ on $K$ by:
\[
  \db(v, w) = \log M(v, w) + \log M(w, v)
\]
\end{remark}

\begin{proposition}[Symmetrisation of $\dh$ by hat map]\label{prop:symmetrization}
For any $X, Y \in \VPM(n)$, we have the following identities:
\begin{align*}
    \dh(X,Y)&=\log M(\hat X,\hat Y)+\log M(\hat Y,\hat X)\\
    M(\hat X,\hat Y)
      &=
      \max\left\{
        \lambda_{\max}(Y^{-1}X),\,
        \lambda_{\max}\bigl((I-Y)^{-1}(I-X)\bigr)
      \right\}
\end{align*}
\end{proposition}
\begin{proof}
    The equation:
    \[
    \dh(X,Y)=\log M(\hat X,\hat Y) + \log M(\hat Y,\hat X)
    \]
    as well as the formula for $M(\hat X,\hat Y)$ are shown in the proof of \cite[Theorem 14]{karwowski2025hilbertgeometrysymmetricpositivedefinite}, although using a quantity $m(X, Y) = \sup\{\lambda > 0 : \lambda Y \preceq X\}$. This is easily converted to $M$ by noting that $m(X, Y) = 1/M(Y, X)$.
\end{proof}

\begin{proposition}\label{prop:spread}
The Hilbert distance $\dh(J_1, J_2)$ for $J_1, J_2 \in \VPM(n)$ can be expressed as a matrix spread:
\begin{align*}
    \dh(J_1, J_2) &= \spread\left(\log \left(\hat{J_2}^{-1/2}\hat{J_1}\hat{J_2}^{-1/2}\right)\right)\\
    &=
\spread\left(\log 
\begin{bmatrix}
    J_2^{-1/2} J_1 J_2^{-1/2} & 0 \\ 
    0 & (I - J_2)^{-1/2}(I - J_1)(I - J_2)^{-1/2}
\end{bmatrix}
\right)
\end{align*}
\end{proposition}
\begin{proof}
From~\Cref{prop:symmetrization}, we know that:
\[
\dh(X,Y)=\log M(\hat{X}, \hat{Y}) + \log M(\hat{Y}, \hat{X})
\]
Let us form a big block diagonal matrix:
\[
P=\begin{bmatrix}
    Y^{-1/2}XY^{-1/2} & 0 \\ 
    0 & (I - Y)^{-1/2}(I - X)(I - Y)^{-1/2}
\end{bmatrix} \in \PD(2n)
\]
which eigenvalues are exactly the union of the eigenvalues of $Y^{-1}X$ and $(I - Y)^{-1}(I - X)$, so:
\begin{align*}
\dh(X,Y)=\log\frac{\lambda_{\max}(P)}{\lambda_{\min}(P)} = \log \lambda_{\max}(P)-\log \lambda_{\min}(P)= \lambda_{\max}(\log P)-\lambda_{\min}(\log P) 
\end{align*}
\end{proof}

\subsection{Finsler norm and Finsler metric}
To report the Finsler structure corresponding to the Hilbert distance, let us first put forward the notion of an asymmetric norm.

\begin{definition}[Asymmetric norm]
    A function $f: \VV \to \mathbb{R}$ on a vector space $V$ is called an asymmetric norm, if it satisfies:
    \begin{itemize}
        \item Triangle inequality: for any $x, y \in \VV$, $f(x + y) \leq f(x) + f(y)$,
        \item Homogeneity: for any $x \in \VV$ and $\lambda \in \mathbb{R}_+$, $f(\lambda x) = \lambda f(x)$,
        \item Positive definiteness: $f(x) \geq 0$ for all $x \in V$, and $f(x) = 0$ if and only if $x = 0$.
    \end{itemize}
\end{definition}

\begin{definition}[Finsler structure]
Let $M$ be a smooth manifold and let $T\mathcal M$ be its tangent bundle. A Finsler structure on $\mathcal M$ is a continuous function $F:T\mathcal M\to[0,\infty)$, such that for every $x\in \mathcal M$, $||\cdot||_{F}: T_x\mathcal M \to \mathbb{R}$ is an asymmetric norm, which we call a Finsler norm on $M$. We define the length functional on piecewise-smooth curves in $M$ as: \[
L_F(\gamma) = \int_a^b ||\dot{\gamma}(t)||^F_x\  dt \qquad \text{ for } \qquad \gamma: [a, b] \to M
\]
and the associated Finsler distance between $x, y \in M$ as: \[
d_F(x, y) = \inf_{\gamma}L_F(\gamma) \qquad \text{ where } \qquad  \gamma(a) = x, \gamma(b) = y. 
\]
\end{definition}

To show how Finsler structure arises from the Hilbert distance, we need to define the exit times, which are related to the Birkhoff ratio $M(\cdot, \cdot)$.

\begin{lemma}[Exit times $t^\pm$]\label{lem:exit-times-M}
Let $K\subseteq \VV$ be an open convex pointed cone. For $w\in K$ and $v\in\VV$ define the forward and backward cone exit times:
\[
t_K^+(w,v)= \sup\{t>0:\ w+tv\in K\}\in(0,\infty] \qquad
t_K^-(w,v)= t_K^+(w,-v)
\]
Then, with the convention $\frac{1}{\infty}=0$, we have:
\[
\frac{1}{t_K^+(w,v)}=M_K(-v,w)
\qquad
\frac{1}{t_K^-(w,v)}=M_K(v,w)
\]
\end{lemma}

\begin{proof}
For $t>0$ we have:
\[
w+tu\in K
\quad\Longleftrightarrow\quad
\frac{1}{t}w+u\in K
\quad\Longleftrightarrow\quad
-u\preceq \frac{1}{t}w.
\]
Hence the set $\{t>0:\ w+tu\in K\}$ corresponds, via $\lambda=\frac{1}{t}$, to the set
$\{\lambda>0:\ -u\preceq \lambda w\}$, and taking the supremum over $t$ is equivalent to
taking the infimum over $\lambda$.
This yields $t_K^+(w,u)=1/M_K(-u,w)$, and the formula for $t_K^-$ follows by replacing
$u\mapsto -u$.
\end{proof}

\begin{proposition}[Finsler norm given by Hilbert distance]\label{prop:finsler}
Let $C\subset \mathbb{R}^n$ be an open bounded convex set equipped with the Hilbert distance $\dh$.
Then $C$ carries a Finsler structure whose norm at $x\in C$, for $v\in T_xC\simeq \mathbb{R}^n$, is given as:
\[
\|v\|^H_x=\frac{1}{t^+(x,v)}+\frac{1}{t^-(x,v)}
\]
\end{proposition}

\begin{proof}
Fix $x\in C$ and $v\neq 0$. Since $C$ is open, convex, and bounded, the line
$\ell=\{x+sv:\ s\in\mathbb{R}\}$ meets $\partial C$ in exactly two points:
\[
a=x-t^-(x,v)v\qquad b=x+t^+(x,v)v
\]
with $t^\pm(x,v)\in(0,\infty)$. For $s\in(-t^-(x,v),t^+(x,v))$ set $y_s=x+sv\in C$.

By the cross-ratio formula for the Hilbert distance on the line $\ell$ (in the coordinate $s$):
\[
\dh(x,y_s)=\log\frac{\|a-y_s\|\,\|b-x\|}{\|a-x\|\,\|b-y_s\|}
=\log\left(\frac{(t^-(x,v)+s)\,t^+(x,v)}{t^-(x,v)\,(t^+(x,v)-s)}\right).
\]
Therefore:
\[
\frac{d}{ds}\,\dh(x,y_s)
=\frac{1}{t^-(x,v)+s}+\frac{1}{t^+(x,v)-s},
\]
and evaluating at $s=0$ gives:
\[
\left.\frac{d}{ds}\right|_{s=0} \dh(x,x+sv)
=\frac{1}{t^-(x,v)}+\frac{1}{t^+(x,v)}.
\]
This defines a Finsler norm on each $T_xC$.
\end{proof}

We would now like to compute the Finsler norm given by the Hilbert distance explicitly, which can be done by moving to the product cone via the hat map. We first need to understand how the Birkhoff ratio behaves on the product cone $\PD^2(n)$.

\begin{lemma}[Birkhoff ratio on a product cone]\label{lemma:product-cone-M}
Let $K \subset \VV$ be an open convex pointed cone inducing an order $\preceq$. Then, for the product cone $K\times K\subset \VV \times \VV$ and its product order $\preceq^2$:
\[
(v, v')\preceq^2 (w, w')
\ \Longleftrightarrow\
v\preceq w \text{ and } v'\preceq w'
\]
we have that:
\[
M((v, v'), (w, w')) = \max(M(v, w), M(v', w'))
\]
\end{lemma}
\begin{proof}
    Unrolling the definition of $M$ on the product cone, we have:
    \[
    M((v, v'), (w, w')) = \inf\{\lambda > 0 : (v, v') \preceq \lambda (w, w')\}
    \]
    meaning that:
    \[
(v, v')\preceq^2 \lambda (w, w')
\ \Longleftrightarrow\
v\preceq \lambda w \text{ and } v'\preceq \lambda w'
\]
so the smallest feasible $\lambda$ is the maximum of the two component $M$.
\end{proof}

\begin{proposition}[Finsler norm from Hilbert distance]\label{prop:finsler-via-M}
Let $\PD^2(n)=\PD(n)\times \PD(n)$ be the product cone, with the product Loewner order. For $X\in \VPM(n)$ and $V\in \Sym(n)$ set $\hat X=(X,I-X)\in {\PD^2(n)}$ and $\hat V=(V,-V)\in \VV$.
Then the Finsler norm induced by $\dh$ at $X$ satisfies:
\begin{align*}
\|V\|_X^{H}
=\max\Big(\max(0,-\lambda_{\min}(A)),\ \max(0,\lambda_{\max}(B))\Big)
+\max\Big(\max(0,\lambda_{\max}(A)),\ \max(0,-\lambda_{\min}(B))\Big)
\end{align*}
where:
\[
A=X^{-1/2}VX^{-1/2}\qquad B=(I-X)^{-1/2}V(I-X)^{-1/2}.
\]
\end{proposition}

\begin{proof}
We show that:
\begin{align*}
\|V\|_X^{H}
&=\frac{1}{t^+(X,V)}+\frac{1}{t^-(X,V)} \\
&= M_{\PD^2(n)}(\hat V,\hat X)+M_{\PD^2(n)}(-\hat V,\hat X)
\\
&=\max\bigl(M_{\PD}(V,X),\,M_{\PD}(-V,I-X)\bigr)  +  \max\bigl(M_{\PD}(-V,X),\,M_{\PD}(V,I-X)\bigr)\\
&=\max\Big(\max(0,-\lambda_{\min}(A)),\ \max(0,\lambda_{\max}(B))\Big)
+\max\Big(\max(0,\lambda_{\max}(A)),\ \max(0,-\lambda_{\min}(B))\Big)
\end{align*}
The first identity follows from~\Cref{prop:finsler}.
By definition:
\[
t^+(X,V)=\sup\{t>0:\ X+tV\in \VPM(n)\}.
\]
Using the hat embedding, $X+tV\in \VPM(n)$ is equivalent to $\hat X+t\hat V\in {\PD^2(n)}$.
Thus: \[
t^+(X,V)=t_{\PD^2(n)}^+(\hat X,\hat V) \qquad t^-(X,V)=t_{\PD^2(n)}^-(\hat X,\hat V)\]
The second identity then follows from \Cref{lem:exit-times-M}. The third one follows because, from~\Cref{lemma:product-cone-M}, we have:
\[
M_{\PD^2(n)}(\hat V,\hat X)=\max\bigl(M_{\PD}(V,X),\,M_{\PD}(-V,I-X)\bigr)
\]
since $M_{\PD^2(n)}(\cdot, \cdot)$ on a product cone is a maximum of the component $M$'s.
It remains to compute $M_{\PD}(P,Q)$ for $Q\in\PD(n)$ and $P\in\Sym(n)$.
By definition:
\[
M_{\PD}(P,Q)=\inf\{\lambda>0:\ P\preceq \lambda Q\}
=\inf\{\lambda>0:\ \lambda Q-P\succeq 0\}.
\]
Congruence by $Q^{-1/2}$ yields equivalently $\lambda I-Q^{-1/2}PQ^{-1/2}\succeq 0$,
i.e.\ $\lambda\ge \lambda_{\max}(Q^{-1/2}PQ^{-1/2})$.
Taking the infimum over $\lambda>0$ gives:
\[
M_{\PD}(P,Q)=\max\bigl(0,\lambda_{\max}(Q^{-1/2}PQ^{-1/2})\bigr).
\]
Applying this with $(P,Q)=(-V,X)$ and $(P,Q)=(V,I-X)$ yields the expression for $1/t^+(X,V)$,
and similarly for $1/t^-(X,V)$.
\end{proof}

\begin{remark}
For $n \geq 2$, consider normalizing the SPD matrix $X$ as follows: $\tX=\frac{1}{\tr(X)} X$. Then we have $\tr(\tX)=1$, and since the trace is the sum of the positive eigenvalues of $X$ and that $\tX$ is SPD, all its eigenvalues fall in the range $(0,1)$. Thus $\tX$ belong to the VPM.
We can thus define the projective VPM distance:
$$
\tilde{d}_H(X_1,X_2)=d_H(\tX_1,\tX_2).
$$ 
\end{remark}

The Finsler structure allows one to perform calculations of means and medians~\cite{arnaudon2012medians}.

\section{Bilogdet divergence on the VPM   and its Hessian structure}\label{sec:vpmlogdet}

This section defines a dually flat structure on the VPM induced by a strictly convex and differentiable function which we term the bilogdet function.

\begin{definition}
    On the set $\VPM(n)$, we define the bilogdet function $\Psi: \VPM(n) \to [0, \infty]$ as:
    \[
    \Psi_{\bld}(X)= -\log\det X - \log\det(I-X) = \Psi_1(X) + \Psi_2(X)
    \]
\end{definition}

In order to simplify the notation, we write $\Psi$ for short of $\Psi_{\bld}$ in the reminder.

\begin{proposition}
    The function $\Psi(X)$ is strictly convex.
\end{proposition}
\begin{proof}
    From~\cite{boyd2004convex}, we know that the function $\Psi_1(X) = - \log \det X$ is convex. The function $\Psi_2(X) = - \log \det (I - X)$ is also strictly convex, because it is just $\Psi_1 \circ (I - \cdot)$, the latter of which is affine, and composition with affine functions preserves convexity. Finally, a sum of two strictly convex functions is again strictly convex.
\end{proof}

\begin{corollary}
    Hessian of $\Psi$ gives a Riemannian metric on $\VPM(n)$, given as $g_X(V, W) =  \nabla^2_X{\Psi}(V, W)$.
\end{corollary}
\begin{proof}
    We write $g_X(V, V)$ as a sum of squares of non-negative numbers:
    \[
    g_X(V, V) = ||A||^2_F  + ||B||^2_F > 0
    \]
    for $A = X^{-1/2}VX^{-1/2}$ and $B = (I - X)^{-1/2}V(I - X)^{-1/2}$.
\end{proof}

\begin{remark}\label{remark:adapted}
    We have $\Psi(X) \to \infty$ as $X \to \partial \bVPM(n)$, therefore the potential $\Psi$ is adapted to the boundary of $\bVPM(n)$.
\end{remark}

\begin{proposition}\label{prop:barrier-norm}
    The interior set $\VPM(n)$ can be given a Riemannian manifold structure, with the Riemannian metric written as:
    \[
    g_X{\Psi}(V, W) =g_X(V,W)=\mathrm{tr}(X^{-1}VX^{-1}W) + \mathrm{tr}((I - X)^{-1}V(I - X)^{-1}W)
    \]
\end{proposition}
\begin{proof}
The interior of $\VPM(n)$ is open in $\Sym(n)$, therefore inherits the ambient smooth manifold structure. We have:
    \begin{align*}
    d\Psi_1(X)[V] &= d(-\log\det X)[V] = -\mathrm{tr}(X^{-1}V)\\
    d\Psi_2(X)[V] &= d(-\log\det (I - X))[V] = \mathrm{tr}((I - X)^{-1}
    V)
    \end{align*}
And again:
\begin{align*}
    \nabla^2\Psi_1(X)[V,W] &= \mathrm{tr}(X^{-1}VX^{-1}W) \\
    \nabla^2\Psi_2(X)[V,W] &= \mathrm{tr}((I - X)^{-1}V(I - X)^{-1}W)
\end{align*}
So the barrier metric is:
\[
g_X(V, W) = \nabla^2_X{\Psi}(V, W)=\mathrm{tr}(X^{-1}VX^{-1}W) + \mathrm{tr}((I - X)^{-1}V(I - X)^{-1}W)
\]
\end{proof}
\begin{remark}
For $X, Y \in \VPM(n)$ and $V \in T_X\VPM(n)$, we denote the corresponding $\Psi$ metric on $T_X\VPM(n)$, $\Psi$-length of piecewise-smooth curves in $\VPM(n)$, and $\Psi$-distance between points of $\VPM(n)$ as:
\[
||V||^\Psi_X = \sqrt{g^\Psi_X(V,V)} \qquad L_\Psi(\gamma) = \int ||\dot{\gamma}(t)||_{\gamma(t)} dt \qquad \dpsi(X, Y) = \inf_\gamma L_\Psi(\gamma)
\]
\end{remark}

The gradient map induced by the bilogdet generator is
\begin{equation} 
J\mapsto \nabla\Psi_{\bld}(J)=-J^{-1}+(I-J)^{-1}.
\end{equation}

\begin{proposition}\label{prop:logbarrier-ismoetry}
  For any $X \in \VPM(n)$, orthonormal conjugation $(X\mapsto U XU^T)$ for $U \in O(n)$ and inversion $X \mapsto I - X$ are isometries for the log-det barrier metric.
\end{proposition}
\begin{proof}
    First of all, from~\cite[Proposition 4]{karwowski2025hilbertgeometrysymmetricpositivedefinite}, we know that $\VPM(n)$ is invariant under orthonormal conjugation and inversion. The conjugation map:
    \[
    c_U(X) = U X U^T
    \]
    is linear in $X$, so its differential is just equal to itself, $d (c_U)(V) = U V U^T$.
    Now, for the first component of the Riemannian metric $g^1_X(V, W) = \nabla^2_X\Psi_1(V, W)$ we compute:
    \begin{align*}
        g^1_{c_U(X)}(d (c_U)V, d (c_U) W) &= g^1_{U X U^T}(U V U^T, U W U^T) \\
        &= \mathrm{tr}((U X U^T)^{-1}(U V U^T)(U X U^T)^{-1}(U W U^T)) \\
        &= \mathrm{tr}(U X^{-1} U^T U V U^T U X^{-1} U^T U W U^T) \\
        &= \mathrm{tr}(U X^{-1} V X^{-1} W U^T) \\
        &= \mathrm{tr}(X^{-1} V X^{-1} W) = g_X(V, W)
    \end{align*}
    where the last line comes from the cyclic invariance of trace. Similarly we compute the second part $g^2_X(V, W) = \nabla^2_X\Psi_2(V, W)$.
    
    For the second class of isometries, the differential of the affine inversion map $X \mapsto i(X) = I - X$ is:
    \[
    di_X V = -V
    \]
    and therefore we compute for both parts of the barrier metric:
    \begin{align*}
    g_{i(X)}(d (i_X)V, d i W) &= g^1_{i(X)}(d (i_X)V, d (i_X) W) + g^2_{i(X)}(d (i_X)V, d (i_X) W) \\
        &= g^1_{(I - X)}(-V, -W) + g^2_{(I - X)}(-V, -W) \\
        &= \mathrm{tr}((I - X)^{-1} (-V) (I - X)^{-1} (-W)) + \mathrm{tr}((I - I + X)^{-1} (-V) (I - I + X)^{-1} (-W)) \\
        &= g^2_X(V, W) + g^1_X(V, W) = g_X(V, W)    
    \end{align*}
\end{proof}


The potential functions $\Psi_1=\Psi_\ld$, $\Psi_2$ and $\Psi_\bld$ induce the following Bregman logdet, complement logdet, and bilogdet divergences, respectively:
\begin{eqnarray*}
B_{\Psi_\ld}(V_1:V_2) &=& \tr(V_1 V_2^{-1})-\log\det(V_1 V_2^{-1})-n,\\
B_{\Psi_2}(J_1:J_2) &=& \tr((I-J_1) (I-J_2)^{-1})-\log\det((I-J_1) (I-J_2)^{-1})-n=B_{\Psi_\ld}(I-J_1:I-J_2), \\
B_{\Psi_\bld}(J_1:J_2) &=& B_{\Psi_\ld}(J_1:J_2)+ B_{\Psi_\ld}(I-J_1:I-J_2)\\
\end{eqnarray*}

Notice that since $\log\det(V_1 V_2^{-1})=\tr(\log(V_1 V_2^{-1}))$, the logdet divergence $B_{\Psi_\ld}(V_1:V_2)$ can be rewritten as:
$$
B_{\Psi_\ld}(V_1:V_2) = \tr(V_1 V_2^{-1}-\log(V_1 V_2^{-1}))-n=\sum_{\lambda_i\in \Lambda(V_1 V_2^{-1})} \lambda_i-\log\lambda_i-1,
$$
where $\Lambda(X)=\{\lambda_1(X),\ldots,\lambda_n(X)\}$ is the eigenspectrum of $X$.
That is, the logdet divergence is a Bregman spectral divergence (and also the AIRM distance).

In general, we have the Riemannian distance $\rho_F$ induced by a Hessian metric $g(\theta)=\nabla^2 F(\theta)$ which is upper bounded by $\sqrt{S_F}$ where $S_F$ is the symmetrized Bregman divergence~\cite{nielsen2023fisher}: $\rho_F(\theta_1,\theta_2)\leq \sqrt{(\theta_2-\theta_1)^\top (\nabla F(\theta_2)-\nabla F(\theta_1))}$.

\begin{definition}[Self-concordant functions]
    A convex function $f: C \to \mathbb{R}$ on an open convex domain $C \subseteq \mathbb{R}$ is called a self-concordant, if:
		$$
    |f'''(x)| \leq 2\, (f''(x))^{3/2}.
    $$
    A convex function $f: C \to \mathbb{R}$ on an open convex domain $C \subseteq \mathbb{R}^n$ is called a self-concordant~\cite{polyak2016legendre} if for any straight line $\ell \in \mathbb{R}^n$, its $1$-dimensional restriction $f|_{\ell \cap C}$ is self-concordant.
\end{definition}

\begin{proposition}[{\cite[Chapter 9]{boyd2004convex}}]\label{prop:properties-self-conc}
We have:
    \begin{itemize}
        \item Self-concordant functions are closed under pre-composition with affine functions and under addition,
        \item The log-determinant function $f(X) = -\log \det X$ is self-concordant.
    \end{itemize}
\end{proposition}

\begin{corollary}\label{prop:concordant}
    The log-barrier function $\Psi$ is self-concordant on $\VPM(n) \subseteq \Sym(n)$.
\end{corollary}
\begin{proof}
		Follows directly from~\Cref{prop:properties-self-conc}, as well as noting that domain $\VPM(n)$ is convex, as proved in in~\cite[Proposition 3]{karwowski2025hilbertgeometrysymmetricpositivedefinite}.
\end{proof}

%

\section{Inequalities between AIRM/logdet dissimilarities and Hilbert/bilogdet VPM dissimilarities}\label{sec:compare}


\subsection{Bounding with the AIRM}

\begin{definition}[Affine-Invariant Riemannian distance (AIRM) on $\PD(n)$]
    For two matrices $P, Q \in \PD(n)$, we define their affine-invariant Riemannian distance (AIRM) as:
    \[
        \dairm(P, Q) = \sqrt{\sum_{i = 1}^n \log^2 \lambda_i(Q^{-1}P)}.
    \]
\end{definition}

We seek bounds on the $\dh$ using $\dairm$. 
There are a few cases to consider. One option is to simply restrict the AIRM from $\PD(n)$ to $\VPM(n)$; another is to push it forward via the James' map $\iota$ (equivalently, pull back the Hilbert metric from $\VPM(n)$ onto $\PD(n)$ via $\iota^{-1}$).

\begin{definition}
    On the set $\VPM(n)$ we define the restricted and the pushed-forward affine-invariant Riemannian distance respectively as:
\begin{eqnarray}
        \dairmr(X, Y) &=& \dairm(X, Y),\\
        \dairmp(X, Y) &=& \dairmp(\iota^{-1}(X), \iota^{-1}(Y)) = \dairm(X(I - X)^{-1}, Y(I - Y)^{-1}).
        \end{eqnarray}
\end{definition}

\subsection{Bounds with restricted AIRM distance $\dairmr$}

\begin{proposition}[No upper bound with $\dairmr$]
    There does not exist a constant $\kappa > 0$ such that for all $X, Y \in \VPM(n)$ we would have an upper bound on $\dh$ via:
    \[
        \dh(X, Y) \leq \kappa\cdot \dairmr(X, Y)
    \]
\end{proposition}
\begin{proof}
    Let us prove that there exists a sequence of matrices $X_t, Y_t$ for $t = 3, 4, \ldots$, such that, simultaneously:
    \[
        \lim_{t \to \infty} \dairmr(X_t, Y_t) = 0, \qquad \lim_{t \to \infty} \dh(X_t, Y_t) > 0
    \]
    which immediately implies the thesis. For $t \geq 3$ define two matrices:
    \[
        X_t = \diag\left(1 - \frac{2}{t}, \frac{1}{2}, \ldots \frac{1}{2}\right) \qquad 
        Y_t = \diag\left(1 - \frac{1}{t}, \frac{1}{2}, \ldots \frac{1}{2}\right)
    \]
    Both matrices lie in $\VPM(n)$ and are diagonal in the same basis, so we can directly compute eigenvalues $\lambda$ of $Y_t^{-1}X_t$:
    \[
    \lambda_1 = \frac{1 - \frac{2}{t}}{1 - \frac{1}{t}}, \qquad \lambda_{2 \leq i \leq n} = 1
    \]
    We note that $\lambda_1 \leq 1$ for all $t > 2$, which will be important later. For now, we compute:
    \begin{align*}
    \lim_{t \to \infty} \dairmr(X_t, Y_t) &= \lim_{t \to \infty} \sqrt{\sum_{i = 1}^n \log^2 \lambda_i(Y_t^{-1}X_t)} \\
    &\leq \lim_{t \to \infty} \sqrt{\log^2 \left(\frac{1 - \frac{2}{t}}{1 - \frac{1}{t}}\right)} \\
    &= \lim_{t \to \infty} \left|\log \left(1 + \frac{1}{1 - t}\right)\right| \\
    &= \lim_{t \to \infty} \left|\frac{1}{1 - t} + o\left(\frac{1}{t^2}\right)\right| = 0
    \end{align*}
    where the last line is the Taylor expansion of $\log(1 + x) = x + O(x^2)$. On the other hand, we compute the Hilbert distance by looking at the matrix:
    \[
        (I - Y_t)^{-1}(I - X_t) =  \diag\left(t, 2, \ldots 2\right) \cdot \diag\left(\frac{2}{t}, \frac{1}{2}, \ldots \frac{1}{2}\right) = \diag\left(2, 1 \ldots 1\right)
    \]
    And thus computing:
    \begin{align*}
    \lim_{t \to \infty} \dh(X_t, Y_t) &= \lim_{t \to \infty} \log \frac{\max\left(\lambda_{\max}(Y_t^{-1}X_t), \lambda_{\max}\!\big((I - Y_t)^{-1}(I - X_t)\big)\right)}{\min\left(\lambda_{\min}(Y_t^{-1}X_t),  \lambda_{\min}\!\big((I - Y_t)^{-1}(I - X_t)\big)\right)}\\
    &= \log \frac{\max\left(1, 2\right)}{\min\left(\lim_{t \to \infty} \frac{1 - 2/t}{1 - 1/t},  1\right)}\\
    &= \log 2 > 0
    \end{align*}
\end{proof}


\begin{theorem}[Tight lower bound on $\dh$ with $\dairmr$]\label{thm:lower-dh-dairmr}
    For any $X, Y \in \VPM(n)$ we have a tight bound:
    \[
    \frac{1}{\sqrt{n}}\dairmr(X, Y) \leq \dh(X, Y)
    \]
\end{theorem}
\begin{proof}
    We first show that the bound is tight. Let us pick some $\epsilon > 0, c > 1$ such that $0 < \epsilon < \frac{1}{c}$, and define:
    \[
    X_\epsilon = \epsilon I \qquad  Y_\epsilon = \epsilon c I
    \]
    We compute the affine-invariant Riemannian distance:
    \[
        \dairmr(X_\epsilon,  Y_\epsilon) = \sqrt{\sum_{i = 1}^n \log^2 \lambda_i(Y_\epsilon^{-1}X_\epsilon )} = \sqrt{n}\log c
    \]
    For the Hilbert distance, we have:
    \[
    \lambda_{\min}(Y_\epsilon^{-1} X_\epsilon) = \lambda_{\max} (Y_\epsilon^{-1} X_\epsilon) = \frac{1}{c} \qquad \lambda_{\min}((I -Y_\epsilon)^{-1}(I -  X_\epsilon)) = \lambda_{\max} ((I - Y_\epsilon)^{-1}(I -  X_\epsilon)) = \frac{1 - \epsilon}{1 - \epsilon c}
    \]
    Therefore, we plug in:
    \begin{align*}
    \lim_{\epsilon \to 0} \dh(X_\epsilon,  Y_\epsilon) &= \lim_{\epsilon \to 0} \log \frac{\max\left(\lambda_{\max}(Y_\epsilon^{-1}X_\epsilon), \lambda_{\max}\!\big((I - Y_\epsilon)^{-1}(I - X_\epsilon)\big)\right)}{\min\left(\lambda_{\min}(Y_\epsilon^{-1}X_\epsilon),  \lambda_{\min}\!\big((I - Y_\epsilon)^{-1}(I - X_\epsilon)\big)\right)}\\
    &= \log \frac{\max\left(\frac{1}{c},  \lim_{\epsilon \to 0} \frac{1 - \epsilon}{1 - \epsilon c}\right)}{\min\left(\frac{1}{c},  \lim_{\epsilon \to 0} \frac{1 - \epsilon}{1 - \epsilon c}\right)} = \log \frac{1}{\frac{1}{c}} = \log c
    \end{align*}
    This gives us the desired result: 
    \[
     \lim_{\epsilon \to 0}  \dairmr(X_\epsilon, Y_\epsilon)= \sqrt{n} \log c = \lim_{\epsilon \to 0} \sqrt{n} \dh(X_\epsilon, Y\epsilon) 
    \]

    We now show that the bound holds for all $X, Y \in \VPM(n)$. For the affine-invariant Riemannian distance, we have the upper bound from~\Cref{lemma:matrix-norms-inequalities}:
    \[
    \dairmr(X, Y) = \sqrt{\sum_{i = 1}^{n} \log^2 \lambda_i (Y^{-1}X)} \leq \sqrt{n}\max_{i}|\log \lambda_{i}(Y^{-1}X)|
    \]
    For the Hilbert distance, from the characterisation in~\Cref{prop:symmetrization} we also have:
    \[
    \dh(X, Y) = \log M(\hat{X}, \hat{Y}) + \log M(\hat{Y}, \hat{X})
    \]
    and we compute:
    \[
    \log M(\hat{X}, \hat{Y}) =  \log \max\left(\lambda_{\max}(Y^{-1}X), \lambda_{\max}((I - Y)^{-1}(I - X))\right) \ge \log \lambda_{\max}(Y^{-1}X)
    \]
    and similarly:
    \[
    \log M(\hat{Y}, \hat{X}) \ge \log  \lambda_{\max}(X^{-1}Y) = -\log  \frac{1}{\lambda_{\max}(X^{-1}Y)} = - \log \lambda_{\min}(Y^{-1}X) 
    \]
    Thus, combining those, we have that:
    \begin{align*}
    \dh(X, Y) &\ge \log M(\hat{X}, \hat{Y}) + \log M(\hat{Y}, \hat{X})\\
    &\geq \max\left(\log \lambda_{\max}(Y^{-1}X), -\log\lambda_{\min}(Y^{-1}X)\right) \\
    &= \max_{i}|\log \lambda_{i}(Y^{-1}X)| \\
    &\geq \frac{1}{\sqrt{n}}\dairmr(X, Y) \\
    \end{align*}
    giving the required inequality.
\end{proof}

\subsection{Bounds with pushed-forward AIRM distance $\dairmp$}

Now, in case of $\dairmp$, the situation reverses - we have an upper bound for $\dh$, but no lower bound. We first note that in dimension $1$, the Hilbert distance and the pushed $\dairmp$ distance match.
\begin{remark}
    We have that $\dairmp(X, Y) = \dh(X, Y)$ for $X, Y \in \VPM(1) = (0, 1)$.
\end{remark}

To prove the lack of bound in dimension $n > 1$, we need some closed-form expressions for eigenvalues and trace of certain positive definite matrices in dimension $2$, since this is going to be our counterexample (which we will then extend to dimensions $n > 2$. These are slightly technical.

\begin{lemma}[Eigenvalues of unit-determinant matrix]\label{lemma:cosh-eigenvalues}
    Let $P, Q \in \PD(2)$, such that $\det P = \det Q$. Then, the eigenvalues of $Q^{-1}P$ are given as:
    \[
    \lambda_{\pm} = \exp(\pm \delta) \qquad \delta = \arccosh\left(\frac{1}{2}\tr(Q^{-1}P)\right)
    \]
\end{lemma}
\begin{proof}
    Let us denote $A = Q^{-1}P$. Since $\det P = \det Q$, we have that $\det A = 1$ from the fact that determinant is a monoidal map. Thus, we know that (unordered) eigenvalues of $A$ must be $\lambda, \frac{1}{\lambda}$ for some $\lambda \geq 1$. Define $\delta = \log \lambda > 0$. We use the fact that $\tr(A) = \lambda + \frac{1}{\lambda} = 2\frac{e^{\delta} + e^{-\delta}}{2} = 2\cosh(\delta)$, and $\cosh$ is strictly increasing on $[0, \infty)$ making the inverse unique.
\end{proof}

\begin{lemma}[Conjugation-multiplication formula]\label{lemma:conjugation-formula}
    Let $D = \diag(a, b)$ be a positive-definite diagonal matrix, and let $U = \begin{bmatrix}
        \cos \theta_t & -\sin \theta_t \\
        \sin \theta_t & \cos \theta_t
    \end{bmatrix}$ be a rotation matrix by $\theta$. Then, we have the following formula:
    \[
        \frac{1}{2}\tr(U^T D^{-1} U D) = \cos^2\theta + \frac{\sin^2\theta}{2}\left(\frac{a}{b} + \frac{b}{a}\right)
    \]
\end{lemma}
\begin{proof}
    \begin{align*}
    \frac{1}{2} \tr(U^T D^{-1} U D) &=
    \frac{1}{2}\tr\left(\begin{bmatrix}
        \cos \theta & -\sin \theta \\
        \sin \theta & \cos \theta
    \end{bmatrix}
    \begin{bmatrix}
        \frac{1}{a} & 0 \\
        0 & \frac{1}{b}
    \end{bmatrix}
    \begin{bmatrix}
        \cos \theta & \sin \theta \\
        -\sin \theta & \cos \theta
    \end{bmatrix}
    \begin{bmatrix}
        a & 0 \\
        0 & b
    \end{bmatrix} \right)\\
    &= \frac{1}{2}\tr\left(
    \begin{bmatrix}
    \cos^2 \theta\frac{1}{a}+\sin^2 \theta \frac{1}{b} & \cos \theta \sin \theta\left(\frac{1}{a}-\frac{1}{b}\right) \\
    \cos \theta \sin \theta\left(\frac{1}{a}-\frac{1}{b}\right) & \sin^2 \theta\frac{1}{a}+\cos^2 \theta \frac{1}{b}
    \end{bmatrix}
    \begin{bmatrix}
        a & 0 \\
        0 & b
    \end{bmatrix} \right)\\
    &= \cos^2\theta + \frac{\sin^2\theta}{2} \left(\frac{a}{b}+ \frac{b}{a}\right)
    \end{align*}
\end{proof}

\begin{lemma}\label{lemma:ba-ibia}
    Let $X, Y \in \VPM(2)$ satisfy $\tr(X) = \tr(Y) = 1$ and $\det(X) = \det(Y)$. Then both $Y^{-1}X$ and $(I - Y)^{-1}(I - X)$ have the same spectrum.
\end{lemma}
\begin{proof}
    By Cayley-Hamilton theorem, every $2 \times 2$ matrix $M$ satisfies its characteristic polynomial:
    \[
    M^2 - \tr(M) M + (\det M)I = 0
    \]
    We apply this to $X$ and $Y$, taking into account that $\tr(X) = \tr(Y) = 1$:
    \[
    X^2 - X + (\det X)I = 0 \qquad Y^2 - Y + (\det Y)I = 0 
    \]
    and therefore rearranging:
    \[
    I - X = (\det X)X^{-1} \qquad I - Y = (\det Y)Y^{-1}
    \]
    Thus, we have:
    \[
    (I - Y)^{-1}(I - X) = \frac{1}{\det Y}Y (\det X)X^{-1} = YX^{-1}
    \]
    since their determinants were assumed equal. But now, because we have $\det YX^{-1} = 1$, and we are in $\Sym(2)$, this matrix has two eigenvalues $\{\lambda, 1/\lambda\}$. Thus, the inverse, $Y^{-1}X$ also has the same eigenvalues $\{1/\lambda, \lambda\}$.
\end{proof}

\begin{proposition}[No lower bound with $\dairmp$]
    There does not exist a constant $\kappa > 0$ such that for all $X, Y \in \VPM(n)$ we would have an lower bound on $\dh$ via:
    \[
        \kappa \cdot \dairmp(X, Y) \leq \dh(X, Y)
    \]
\end{proposition}
\begin{proof}
    As before, we exhibit a sequence of elements $P_t, Q_t \in \PD(n)$ such that: \[
    \lim_{t \to \infty} \dairm(P_t, Q_t) = \lim_{t \to \infty}  \dairmp(\iota(P_t), \iota(Q_t)) > 0 \qquad \lim_{t \to \infty} \dh(\iota(P_t), \iota(Q_t)) = 0
    \]
    As promised, we give the sequence for dimension $n=2$ first, and generalise to arbitrary $n$ at the end. Let us define the following matrices:
    \[
    P_t = \begin{bmatrix}
        t & 0 \\
        0 & \frac{1}{t}
    \end{bmatrix}
    \qquad 
    Q_t = U_t P_t U_t^T
    \qquad
    U_t = \begin{bmatrix}
        \cos \theta_t & -\sin \theta_t \\
        \sin \theta_t & \cos \theta_t
    \end{bmatrix}
    \qquad \theta_t = \frac{1}{t}
    \]
    Because $P_t$ and $Q_t$ only differ by rotation conjugation, they have the same determinant. Thus, from~\Cref{lemma:cosh-eigenvalues}, we know that the eigenvalues of $Q_t^{-1}X$ are of the form: \[
    \lambda_{\pm} = \exp\left(\pm \arccosh\frac{1}{2}\tr(Q_t^{-1}X)\right)
    \]
    Using the formula from~\Cref{lemma:conjugation-formula} and plugging in the diagonal matrix $P_t = \diag(t, 1/t)$ we have:
    \begin{align*}
    \frac{1}{2}\tr(Q_t^{-1}X) &= \frac{1}{2} \tr(U_t X^{-1}_t U_t^T P_t)\\
    &= \cos^2\theta_t + \frac{\sin^2\theta_t}{2} \left(t^2 + \frac{1}{t^2}\right)\\
    &= (1 - \sin^2\theta_t) + \frac{\sin^2\theta_t}{2}\left(t^2 + \frac{1}{t^2}\right) \\
    &= 1 + \frac{\sin^2\frac{1}{t}}{2}\left(t^2 + \frac{1}{t^2} - 1\right)
    \end{align*}
    From Taylor expansion, we know that $\sin^2 \frac{1}{t} =\frac{1}{t^2} + O\left(\frac{1}{t^4}\right)$, giving:
    \begin{align*}
    \lim_{t \to \infty} 1 + \frac{\sin^2\theta_t}{2}\left(t^2 + \frac{1}{t^2} - 1\right) &= \lim_{t \to \infty} 1 +\frac{1}{2} \left(\frac{1}{t^2} + o\left(\frac{1}{t^4}\right)\right)\left(t^2 + \frac{1}{t^2} - 1\right) = \frac{3}{2}
    \end{align*}
    Thus, we plug into the formula for AIRM to get:
    \[
        \dairmr(P_t, Q_t) = \sqrt{\log^2 \lambda_{+} + \log^2 (1/\lambda_{+})} = \sqrt{2}\left|\arccosh\left(\frac{3}{2}\right)\right| > 0
    \]
    Let us denote the matrices $P_t, Q_t$ pushed via James' map:
    \[
    X_t = \iota(P_t) = P_t(I + P_t)^{-1} \qquad Y_t = \iota(Q_t) = Q_t(I + Q_t)^{-1}
    \]
    From definition, $\dairmp(X_t, Y_t) = \dairmr(P_t, Q_t)$. 
    
    On the other hand, we now explicitly compute Hilbert distance for $X_t, Y_t$. First:
    \[
    X_t =\begin{bmatrix}
        t &  0 \\
        0 & \frac{1}{t}
    \end{bmatrix}
    \begin{bmatrix}
        \frac{1}{1 + t} &  0 \\
        0 & \frac{t}{t + 1}
    \end{bmatrix}
    = \begin{bmatrix}
        \frac{t}{1 + t} &  0 \\
        0 & \frac{1}{t + 1}
    \end{bmatrix}
    \]
    From this direct calculation, we see that $\tr(X_t) = 1$. From~\cite[Remark 20]{karwowski2025hilbertgeometrysymmetricpositivedefinite}, we know that rotation can be exchanged with the James' map, giving $Y_t = U_tX_tU_t^T$, and the same is true for inversion, giving $I - Y_t = I - U_t X_t U_t^T = U_t(I - X_t)U_t^T$. Since they differ by rotation conjugation, they have the same determinant and the same trace, which we use to apply~\Cref{lemma:ba-ibia} and obtain that spectra of $(I - Y_t)^{-1}(I - X_t))$ and $\lambda_\pm(Y_t^{-1}X_t$ are equal. From~\Cref{lemma:cosh-eigenvalues} their eigenvalues are given:
    \[
    \lambda_\pm((I - Y_t)^{-1}(I - X_t)) = \lambda_\pm(Y_t^{-1}X_t) = e^{\pm\delta_t} \qquad \delta_t = {\arccosh\left(\frac{1}{2}\tr(Y_t^{-1}X_t)\right)} > 0
    \]

    Let us now look at the $\VPM$ distance formula:
    \[
    \dh(X_t, Y_t) = \log \frac{\max\left(\lambda_{\max}(Y_t^{-1}X_t), \lambda_{\max}\!\big((I - Y_t)^{-1}(I - X_t)\big)\right)}{\min\left(\lambda_{\min}(Y_t^{-1}X_t),  \lambda_{\min}\!\big((I - Y_t)^{-1}(I - X_t)\big)\right)}
    \]
    Because spectra of $(I - Y_t)^{-1}(I - X_t)$ and $Y_t^{-1}X_t$ are equal, the factors in $\min, \max$ are the same, and we can simplify the ratio to:
    \[
    \dh(X_t, Y_t) = \log \frac{e^{\delta_t}}{e^{-\delta_t}} = 2\delta_t
    \]
    The last step is then to show that $ \lim_{t\to \infty}\delta_t = 0$:
    \begin{align*}
        \lim_{t \to \infty} \arccosh\left(1 + \frac{\sin^2\theta_t}{2}\left(t + \frac{1}{t} - 1\right)\right)
        &= \arccosh\left(1 + \lim_{t \to \infty} \frac{1}{2}\left(\frac{1}{t^2} + O\left(\frac{1}{t^4}\right)\right)\left(t + \frac{1}{t} - 1\right)\right)\\
        &= \arccosh(1) = 0
    \end{align*}
    Finally, to extend the argument for general $n > 2$, we define the block-diagonal matrices: \[
    A^n_t = X_t \oplus \frac{1}{2}I_{n-2} \qquad B^n_t = Y_t \oplus \frac{1}{2}I_{n-2}
    \]
    which then $A^n_t, B^n_t \in \VPM(n)$. For $\dairmp$, these additonal dimensions do not contribute to the sum (since they cancel out leaving zero after the logarithm), thus $\lim_{t \to \infty} \dairmp(A^n_t, B^n_t) = \lim_{t \to \infty} \dairmp(X_t, Y_t) > 0 $. For Hilbert distance, we have: \[
    {(B^n_t)}^{-1}A^n_t = (Y_t^{-1}X_t)\oplus I_{n-2} \qquad (I - B^n_t)^{-1}(I - A^n_t) = (I - Y_t)^{-1}(I - X_t)\oplus I_{n-2}
    \]
    Since $\delta_t > 0$, we have $e^{\delta_t} \geq 1 \geq e^{-\delta_t}$, so the minimum and maximum are still attained on the same elements as before, making $\dh$ unchanged too.
\end{proof}

In the other direction, the proof relies on the composition of a few maps which we separately prove to be Lipschitz with explicit constants. In particular, we will prove it twice for the hat embedding map from~\Cref{def:hat}, once when $\VPM$ is given the Hilbert distance, and second time when it's given log-barrier distance.

\begin{lemma}[Hat embedding is $\sqrt{2}$-co-Lipschitz]\label{lemma:composite}
The hat embedding $(\hat{\phantom{x}}): \VPM(n) \to \PD(n) \times \PD(n)$ is $\sqrt{2}$-co-Lipschitz, if we take the domain $\VPM(n)$ to be equipped with Hilbert distance, and the codomain $\PD(n) \times \PD$ equipped with the product AIRM distance.
\end{lemma}
\begin{proof}
For any $X, Y \in \VPM(n)$, we have to prove that:
\[
\dh(X, Y) \leq \sqrt{2} \sqrt{\left(\dairmr(X,Y)\right)^2+\left(\dairmr(I-X,I-Y)\right)^2}
\]
From~\Cref{prop:spread}, we know that:
\[
\dh(X,Y)=\spread(\log P)
\]
for the matrix:
\[
P=\begin{bmatrix}
    Y^{-1}X & 0 \\ 
    0 & (I - Y)^{-1}(I - X)
\end{bmatrix} \in \PD(2n)
\]
Therefore, we can apply range-$l_2$ inequality (\Cref{lemma:range-l2}) to the eigenvalue vector of $\log P$. Since $\log P$ is symmetric and $||\log P||_F^2$ is the sum of squares of its eigenvalues, we get:
\begin{align*}
\dh^2(X,Y) \le 2 ||\log P||^2_F &= 2(||\log Y^{-1}X||_F^2+||\log (I - Y)^{-1}(I - X)||_F^2) \\
&= 2\left(\left(\dairmr(X,Y)\right)^2+\left(\dairmr(I-X,I-Y)\right)^2\right)
\end{align*}
\end{proof}


\begin{lemma}[Infinitesimal inequality implies $\kappa$-Lipschitz]\label{lemma:pointwise-to-global}
Let $F: \mathcal M\to \mathcal N$ be a smooth map between Finsler manifolds $(\mathcal M,g)$ and $(\mathcal N,h)$ with induced distances $d_{g}$ and $d_{h}$ respectively, and assume that there exists a constant $\kappa> 0$ such that for every $p\in \mathcal M$ and every $v\in T_p \mathcal M$ we have the pointwise inequality:
\[
||dF_p(v)||_{F(p)} \leq \kappa||v||_{p}
\]
Then $F$ is $\kappa$-Lipschitz, that is, for all $x, y \in \mathcal M$, we have:
\[
d_h(F(x),F(y)) \le \kappa d_g(x,y)
\]
\end{lemma}

\begin{proof}
Let $\gamma:[0,1]\to\mathcal M$ be a piecewise $C^1$ curve. By the chain rule,
$(F\circ\gamma)'(t)=dF_{\gamma(t)}(\gamma'(t))$, so by assumption:
\[
||(F\circ\gamma)'(t)||_{F(\gamma(t))} \leq \kappa ||\gamma'(t)||_{\gamma(t)}.
\]
Integrating in $t$ yields $L_h(F\circ\gamma)\leq \kappa L_g(\gamma)$. Taking the infimum over all
piecewise $C^1$ curves $\gamma$ joining $x$ to $y$ gives
$d_{h}(F(x),F(y))\leq \kappa d_{g}(x,y)$.
\end{proof}

\begin{lemma}[Hat embedding is an isometric embedding]
Considering $\VPM(n)$ equipped with the log-barrier metric $\dpsi$, and $\PD(n) \times \PD(n)$ equipped with the product AIRM metric, the hat embedding is an isometric immersion.
\end{lemma}

\begin{proof}
Let us denote the product AIRM Riemannian metric by $g^\times=  \gairm \times \gairm$. The map $(\hat{\phantom{x}})$ is affine, hence $d\hat{X}(V)=(V,-V)$. Therefore:
\begin{align*}
g^\times_{\hat{X}}\big(d\hat{X}(V),d\hat{X}(W)\big) 
&= \gairm_X(V,W)+\gairm_{I-X}(-V,-W)\\
&= \tr(X^{-1}VX^{-1}W)+\tr((I-X)^{-1}V(I-X)^{-1}W)
\end{align*}
which is exactly $g^\Psi_X(V,W)$. The $1$-Lipschitz statement then follows from~\Cref{lemma:pointwise-to-global}.
\end{proof}

\begin{corollary}[The hat embedding is $1$-Lipschitz]\label{lem:G-isometric}
Considering $\VPM(n)$ equipped with the log-barrier metric $\dpsi$, and $\PD(n) \times \PD(n)$ equipped with the product AIRM metric, the hat embedding is $1$-Lipschitz.
\end{corollary}

\begin{lemma}[Computation of AIRM in coordinates]\label{lemma:mairm-in-coord}
    Let $P = \diag(p_1, \ldots p_n) \in \PD(n)$ and $V  = (v_{ij}) \in \Sym(n)$. Then:
    \[
    \tr(P^{-1}VP^{-1}V) = \sum_{i,j}\frac{v_{ij}^2}{p_i p_j}
    \]
\end{lemma}
\begin{proof}
We write the matrix $P^{-1}VP^{-1}$ in diagonal $P$ basis:
\[
(P^{-1}VP^{-1})_{ij}=\sum_{k,l}(P^{-1})_{ik}V_{kl}(P^{-1})_{l j}
= p_i^{-1}\, v_{ij}\, p_j^{-1}=\frac{v_{ij}}{p_i p_j}.
\]
Now compute the trace:
\[
\tr(P^{-1}VP^{-1}V)=\sum_i ((P^{-1}VP^{-1})V)_{ii} = \sum_{i,j}(P^{-1}VP^{-1})_{ij}V_{ji} = \sum_{i,j}\frac{v_{ij}}{p_i p_j}v_{ji} = \sum_{i,j}\frac{v^2_{ij}}{p_i p_j}
\]
\end{proof}
\begin{lemma}[The James map is $1$-Lipschitz into the log-barrier geometry]\label{lem:iota-lipschitz}
The James map $\iota:\PD(n)\to \VPM(n)$ is $1$-Lipschitz, where we consider $\PD(n)$ with $\dairm$ distance and $\VPM(n)$ with the log-barrier distance $d_\Psi$.
\end{lemma}
\begin{proof}

Let us fix $P\in\PD(n)$ and $V\in\Sym(n)$. From~\Cref{prop:diota}, we have that:
\[
d\iota_P(V) = AVA \qquad A = (I + P)^{-1}
\]
Let us denote $S=d\iota_P(V)$. Thus, we need to prove the pointwise estimate:
\[
||d\iota_P(V)||^\Psi_{\iota(P)} = ||S||^\Psi_{\iota(P)} \leq \mairm{V}_{P}
\]
and then from~\Cref{lemma:pointwise-to-global},this will imply that $\iota$ is $1$-Lipschitz from $(\PD(n),\dairm)$ to $(\VPM(n),d_\Psi)$.

Both sides of the inequality are invariant under orthogonal conjugation $P \mapsto U^TPU$ and $V \mapsto U^TVU$: both AIRM and $\Psi$ are invariant under conjugation (\Cref{prop:logbarrier-ismoetry}), $\iota(P)$ is equivariant (\cite[Remark 20]{karwowski2025hilbertgeometrysymmetricpositivedefinite}), and $S$ is also equivariant:
\[
U^TSU = (U^TAU)(U^TVU)(U^TAU) = (I+U^TPU)^{-1}(U^TV U)(I+U^TPU)^{-1}
\]

Thus w.l.o.g. $P=\diag(p_1,\dots,p_n)$ with $p_i>0$. Then let $X = \iota(P)$, and:
\[
A=\diag(a_1,\dots,a_n),\quad a_i=\frac{1}{1+p_i},
\qquad
X=\diag(x_1,\dots,x_n),\quad x_i=\frac{p_i}{1+p_i}=1-a_i.
\]
where we also note that $0 < a_i < 1$. Write $V=(v_{ij})$ in this basis, and applying~\Cref{lemma:mairm-in-coord} to the pair $(P, V)$, we get:
\[
\left(\mairm{V}_P\right)^2=\tr(P^{-1}VP^{-1}V)=\sum_{i,j} v^2_{ij}/p_ip_j = \sum_{i,j} b_{ij}
\]
where we defined $b_{ij} = v^2_{ij}/p_ip_j$. 
Moreover $S=AVA$ has entries $s_{ij}=a_i a_j v_{ij}$, so applying~\Cref{lemma:mairm-in-coord} again to the pair $(X, S)$ we obtain:
\[
\tr(X^{-1}SX^{-1}S)=\sum_{i,j}\frac{s_{ij}^2}{x_ix_j}
=\sum_{i,j}\frac{a_i^2a_j^2}{(p_ia_i)(p_ja_j)}\,v_{ij}^2
=\sum_{i,j} a_i a_j\, b_{ij}
\]
and similarly
\[
\tr((I-X)^{-1}S(I-X)^{-1}S)=\tr(A^{-1}SA^{-1}S)=\sum_{i,j}\frac{s_{ij}^2}{a_ia_j}
=\sum_{i,j}(p_ia_i)(p_ja_j)\,b_{ij}
=\sum_{i,j} x_i x_j\, b_{ij}
\]
Therefore, using Cauchy-Schwarz and the fact that $\sqrt{a_i^2 + (1 - a_i)^2} = \sqrt{1 - 2a_i(1 - a_i)} < 1$:
\begin{align*}
    \left(||S||^{\Psi}_X\right)^2  
    &= \sum_{i,j} (a_i a_j + x_i x_j) b_{ij} \\
    &= \sum_{i,j} (a_i a_j + (1 - a_i)(1-a_j) b_{ij} \\
    &= \sum_{i,j} \langle (a_i,1 - a_i),(a_j,1-a_j)\rangle b_{ij} \\
    &\leq \sum_{i,j} ||(a_i,1-a_i)||_2 ||(a_j,1-a_j)||_2 b_{ij} \\
    &\leq \sum_{i,j} \sqrt{a_i^2 + (1-a_i)^2}\sqrt{a_j^2 + (1-a_j)^2} b_{ij} \\
    &\leq \sum_{i,j} b_{ij}= \left(\mairm{V}_P\right)^2
\end{align*}
\end{proof}

\begin{corollary}\label{lemma:1lipschitz}
    Let us consider $\PD(n)$ equipped with $\dairmr$ distance and $\PD(n) \times \PD(n)$ equipped with the corresponding product distance. Then, the map $F: \PD(n) \to \PD(n) \times \PD(n)$ given by $F(P) = (\iota(P), (I + P)^{-1})$ is $1$-Lipschitz.
\end{corollary}

\begin{proof}
By composition of~\Cref{lem:iota-lipschitz} and~\Cref{lem:G-isometric}.
\end{proof}

\begin{theorem}[Tight upper bound of $\dh$ with $\dairmp$]\label{thm:ubairmp}
    For any $X, Y \in \VPM(n)$ we have a tight bound
    \[
    \dh(X, Y) \leq \sqrt{2}\dairmp(X, Y)
    \]
\end{theorem}
\begin{proof}
    Let us write: \[
    P = \iota^{-1}(X) = X(I - X)^{-1}\qquad Q = \iota^{-1}(Y) = Y(I - Y)^{-1}
    \]
    First composing with the inequality from~\Cref{lemma:composite} and then the inequality from~\Cref{lemma:1lipschitz}, we get:
    \begin{align*}
    \frac{1}{\sqrt{2}}\dh(X, Y) 
    &\leq \sqrt{\left(\dairmr(X, Y)\right)^2 + \left(\dairmr(I - X, I - Y)\right)^2}\\
    &= \sqrt{\left(\dairmr(\iota(P), \iota(Q))\right)^2 + \left(\dairmr(I - \iota(P), I - \iota(Q))\right)^2} \\
    &\leq \dairm(P, Q) =\dairm(\iota^{-1}(X), \iota^{-1}(Y)) = \dairmp(X, Y)
    \end{align*}
\end{proof}

\subsection{Bounding with the log-barrier metric}
For the log-barrier metric, we have both the lower and the upper bound.


\begin{theorem}[Tight upper bound with $\|\cdot\|^\Psi$]\label{thm:upper-psi}
    For all $X \in \VPM(n)$ and all $V \in T_X\VPM(n) \simeq \Sym(n)$, we have:
    \[
        \|V\|^H_X \leq \sqrt{2}\, \|V\|^\Psi_X
    \]
\end{theorem}
\begin{proof}
    From~\Cref{prop:barrier-norm}, we have that the barrier metric norm is given by:
    \[
    \left(||V||^\Psi_X\right)^2 = \tr(X^{-1}VX^{-1}V) + \tr((X - I)^{-1}V(X - I)^{-1}V) = ||A||_F^2+||B||_F^2
    \]
    where:
    \[
    A =X^{-1/2}VX^{-1/2} \qquad B=(I - X)^{-1/2}V(I - X)^{-1/2}
    \]
    On the other hand, from~\Cref{prop:finsler}, we write the Hilbert metric on the tangent space $T_X\VPM(n)$ as:
    \[
    ||V||^H_X = \frac{1}{t^+(X,V)}+\frac{1}{t^-(X,V)}
    \]
    where we have:
    \begin{align*}
        \frac{1}{t^+(X,V)}&=\max\left(\max(0,-\lambda_{\min}(A)),\max(0,\lambda_{\max}(B))\right)\\
        \frac{1}{t^-(X,V)}&=\max\left(\max(0,\lambda_{\max}(A)),\max(0,-\lambda_{\min}(B))\right)
    \end{align*}
    which we have shown in~\Cref{prop:finsler-via-M}. Now, let:
    \[
        u =\max(0,-\lambda_{\min}(A)) \qquad p=\max(0,\lambda_{\max}(A)) \qquad
        v =\max(0,\lambda_{\max}(B)) \qquad q=\max(0,-\lambda_{\min}(B))
    \]
    Then, $||V||^H_X= \max(u,v)+\max(p,q)$. Note that we always have:
    \[
    ||A||_F^2=\sum_i \lambda_i(A)^2 \geq \lambda_{\max}(A)^2+\lambda_{\min}(A)^2 \ge p^2+u^2 \qquad ||B||_F^2\geq v^2+q^2
    \]
    Thus, we compare:
    \begin{align*}
        ||V||^H_{X} &= \max(u,v)+\max(p,q) \\
    &\leq \sqrt{u^2+v^2}+\sqrt{p^2+q^2} \\ 
    &\leq \sqrt{2} \sqrt{u^2+v^2+p^2+q^2} \\ 
    &\leq \sqrt{2} \sqrt{||A||_F^2+||B||_F^2} =\sqrt{2}||V||^{\Psi}_X
    \end{align*}
    where we used $\max(a,b) \leq \sqrt{a^2+b^2}$ in the first inequality.

    To show this is tight, we take $X = \frac{1}{2}I_n$, and $V = e_1 e_1^T$. This gives:
    \[
    A = B = 2e_1e_1^T
    \]
    Therefore $||V||^\Psi_X = 2(||A||^\Psi_X) = 2\sqrt{2}$. On the other hand, $A$ and $B$ only have one non-zero eigenvalue $\lambda_{\max} = 2$, so $||V||^H_X = 2+ 2 = 4$. Thus, the ratio is $\sqrt{2}$
\end{proof}

\begin{theorem}[Tight lower bound with $\|\cdot\|^\Psi$]\label{thm:lower-psi}
    For all $X \in \VPM(n)$ and all $V \in T_X\VPM(n) \simeq \Sym(n)$, we have:
    \[
        ||V||^\Psi_X \leq \sqrt{n}||V||^H_X
    \]
\end{theorem}
\begin{proof}
    Similarly to the proof of~\Cref{thm:upper-psi} above, we have:
    \[
    A =X^{-1/2}VX^{-1/2} \qquad B=(I - X)^{-1/2}V(I - X)^{-1/2}
    \]
    and therefore:
    \[
    \left(||V||^\Psi_X\right)^2 = \tr(X^{-1}VX^{-1}V) + \tr((X - I)^{-1}V(X - I)^{-1}V) =||A||_F^2+||B||_F^2
    \]
    As before, we proceed with:
    \[
        u =\max(0,-\lambda_{\min}(A)) \qquad p=\max(0,\lambda_{\max}(A)) \qquad
        v =\max(0,\lambda_{\max}(B)) \qquad q=\max(0,-\lambda_{\min}(B))
    \]
    where $||V||_X^H = \max(u, v) + \max(p, q)$. Let us now consider the number of non-negative/number of negative eigenvalues of $A$, denoted by $r^+$ and $r^-$. We observe that these numbers are the same for $B$, which follows from the Sylvester's law of inertia, because $A$ and $B$ are congruent to the same matrix $V$.
    Therefore, we have:
    \[
        \|A\|_F^2 \leq r^- u^2 + r^+ p^2 \qquad \|B\|_F^2\leq r^+ v^2 + r^- q^2
    \]
    Similarly, we have bounds for the Hilbert norm $||V||^H_X = \max(u, v) + \max(p, q)$:
    \[
        \|V\|^H_X \geq u + p \qquad \|V\|^H_X \geq v + q
    \]
    This gives the final inequality:
    \begin{align*}
        \left(||V||^\Psi_X\right)^2 
        &= ||A||_F^2+||B||_F^2 \\
        &\leq r^- u^2 + r^+ p^2 + r^+ v^2 + r^- q^2 \\ 
        &\leq r^-(u^2 + q^2) + r^+(v^2 + p^2) \\ 
        &\leq r^-(u + q)^2 + r^+(v + p)^2 \\
        &\leq r^-\left(||V||^H_X\right)^2 + r^+\left(||V||^H_X\right)^2 \\
        &\leq (r^- + r^+)\left(||V||^H_X\right)^2 = n\left(||V||^H_X\right)^2
    \end{align*}

    To show it is tight, we let $X = \epsilon I$ for some $\epsilon > 0$, and $V =I$. Then:
    \[
    A = \frac{1}{\epsilon}I \qquad B = \frac{1}{1-\epsilon}I
    \]
    and:
    \begin{align*}
    ||V||^H_X &= \max(u, v) + \max(p, q) = \frac{1}{\epsilon} + \frac{1}{1-\epsilon} = \frac{1}{\epsilon(1-\epsilon)}\\
    ||V||^\Psi_X &= \sqrt{n\frac{1}{\epsilon^2} + n\frac{1}{(1- \epsilon)^2}} =\sqrt{n}\frac{1}{\epsilon(1-\epsilon)}\sqrt{((1-\epsilon)^2 + \epsilon^2)}
    \end{align*}
    In the limit:
    \[
    \lim_{\epsilon \to 0} \frac{||V||^\Psi_X}{||V||^H_X} = \lim_{\epsilon \to 0} \sqrt{n}((1-\epsilon)^2 + \epsilon^2)= \sqrt{n}
    \]
    
\end{proof}

\begin{corollary}[Tight lower and upper bound with $\dpsi$]
    For all $X, Y\in \VPM(n)$, we have:
    \[
        \frac{1}{\sqrt{2}}\dh(X, Y) \leq \dpsi(X, Y) \leq \sqrt{n}\dh(X, Y)
    \]
\end{corollary}
\begin{proof}
    Directly from applying~\Cref{lemma:pointwise-to-global} to~\Cref{thm:upper-psi,thm:lower-psi}.
\end{proof}
\begin{table}[]
    \centering
    \begin{tabular}{l|c|c}
         Distance & Upper bound constant & Lower bound constant \\
         \hline
         $\dairmr$ & \no & $\frac{1}{\sqrt{n}}$ \\
         $\dairmp$ & $\sqrt{2}$ & \no \\
         $\dpsi$ & $\sqrt{2}$ & $\frac{1}{\sqrt{n}}$
    \end{tabular}
    \caption{Summary of lower/upper bounds for the Hilbert distance. Constants are written with respect to $\dh$. All bounds are tight.}
    \label{tab:distances}
\end{table}

\section{Summary and discussion}\label{sec:concl}

In this work, we have  reported two new differential structures on the SPD cone manifold by considering  James' bicone domain~\cite{james1973variance}   in the first part: Namely, the Finsler structure induced by the Hilbert distance on the bicone and the dually flat structure~\cite{IG-2016} induced by the bilogdet function, a natural generalization of the cone logdet function onto James' bicone domain.
In particular, we have shown that the Hilbert VPM distance generalizes the Hilbert simplex distance~\cite{nielsen2023non} since the open spectraplex is a subset of the bicone domain.
Furthermore, in Hilbert's underlying Finsler geometry, the open spectraplex is a totally geodesic submanifold. We reported the explicit constant-speed geodesic parameterization of Hilbert geodesics in both the standard simplex and the VPM domains. 
In the second part, we have compared the affine-invariant Riemannian metric (AIRM or trace metric) distance and the Bregman logdet divergence with the newly defined VPM Hilbert distance and VPM Bregman bilogdet  divergence.

Defining and studying dissimilarities in James' VPM domain of the SPD cone manifold is promising for many applications. 
The Hilbert VPM distance relies only on extremal eigenvalues may be interpreted as  the worst-direction distortion metric, and is thus fundamentally different from affine-invariant or log-Euclidean geometry.
For example, the VPM Hilbert distance may be useful 
in robust control and Lyapunov theory or Ricatti equations where the mapping $P(X)$ normalizes eigenvalues in $(0,1)$, see~\cite{lancaster1995algebraic,jedra2022minimal}.
Another natural application domain is quantum information theory~\cite{nielsen2010quantum} where the effect matrices fall in the closed VPM domain~\cite{montiel2025foundations}.
The fact that the Hilbert VPM distance tends to infinity either as $J\rightarrow 0$ or $J\rightarrow I$ is important in optimization using barrier geometry. This property is due to the invariance under complement operation $J\rightarrow I-J$.
The novel VPM barrier bilogdet function and its underlying dual information geometry may thus proved useful in optimization theory~\cite{ohara2024doubly}.

The web page for the Variance-Precision Model/Manifold of the SPD cone is \url{https://franknielsen.github.io/VPM/index.html}

\appendix
\section{Notations}\label{sec:notations}
\begin{table}[h]
    \caption{Notations}
    \label{tab:placeholder}
    \centering
    \begin{tabular}{ll}
         Symbols & Domain \\
         \hline
         $U \in O(n)$ & Orthonormal matrix  \\ 
         $(\cdot)^T$ & Matrix transpose \\
         $X, Y, Z \in \VPM(n)$ &  Elements of Variance-Precision manifold \\
         $A, B, V, W \in \Sym(n)$ & Symmetric matrices (elements of the tangent space) \\
         $P, Q \in \PD(n)$ & Positive-definite matrices \\
         $\hat{X} = (X, I - X)$ & Pairing embedding \\ 
         $\lambda_{\min}(V), \lambda_{\max}(V), \lambda_i(V)$ & Smallest/largest/$i$-th largest eigenvalue of a symmetric matrix $V$ \\ 
         $\VV, \WW$ & Finite dimensional vector spaces \\
         $u, v, x, y, z \in \mathbb{R}^n$ & Real vectors \\ 
         $s, t, \lambda, \mu, \kappa, \theta \in \mathbb{R}$ & Real numbers \\
         $i, j, k, l \in \mathbb{N}$ & Natural numbers \\ 
         $C \subseteq \mathbb{R}^n$ & Open bounded convex subset of $\mathbb{R}^n$ \\ 
         $\dh, ||\cdot||^H$ & Hilbert distance, corresponding Finsler norm \\
         $\dairm, \mairm{\cdot}, \gairm$ & AIRM distance, corresponding Riemannian norm and Riemannian metric \\
         $\dairmr$ & AIRM restricted distance \\
         $\dairmp$ & AIRM pulled distance \\
         $\mathcal M, \mathcal N$ & Smooth manifolds \\
         $T\mathcal M, T_p\mathcal M$ & Tangent bundle, tangent space at $p \in \mathcal M$ \\
         $K$ & Open convex pointed cone in a vector space \\ 
         $\gamma(t): [0, 1] \to \mathcal M$ & A piecewise-smooth curve in $\mathcal M$\\ 
         $L_g(\gamma)$ & Length functional of a curve $\gamma$ under a Riemannian metric $g$\\ 
         $D = \diag(d_1, \ldots d_n)$ & Diagonal matrix of entries $d_1, \ldots d_n$\\ 
    \end{tabular}

\end{table}

\section{Symbolic calculation of Hilbert VPM distances}\label{sec:maxima}

We use Maxima\footnote{\url{https://maxima.sourceforge.io/}} symbolic package~\cite{calvo2018scientific} to calculate the Hilbert VPM distance and check various properties.
Below is a code snippet

\begin{verbatim}
/* Maxima code for Hilbert VPM distance of 2x2 SPD matrices */
kill(all);
I : matrix([1, 0],  [0, 1]);
/* Check whether a matrix is symmetric positive definite */
isSPD(M) := block(
    [vals],
    vals : eigenvalues(M)[1],    
    if every(lambda([x], is(x > 0)), vals) then
        return(true)
    else
        return(false)
)$

/* Generate random 2x2 symmetric positive-definite matrix */
randomSPD() := block(
    [B, A],
    B : matrix([random(1.0)+1, random(1.0)+1],[random(1.0)+1, random(1.0)+1]),  
    A : transpose(B) . B,
    return(A)
)$

HilbertVPMDistance(J1,J2):= block(
[T1, T2, lambda1, lambda2],
T1: invert(J2).J1,
T2: invert(I-J2).(I-J1),
lambda1: sort(eigenvalues(T1)[1]),
lambda2: sort(eigenvalues(T2)[1]),
return(log(max(lambda1[2],lambda2[2])/min(lambda1[1],lambda2[1])))
)$


J1: matrix(
        [7/20,      -3*sqrt(3)/20],
        [-3*sqrt(3)/20, 13/20]
     )$
J2: matrix(
        [11/20,     -sqrt(3)/20],
        [-sqrt(3)/20, 9/20]
     )$
	 
HilbertVPMDistance(J1,J2);
float(%);

iota(X):=X.invert(I+X);

V1: randomSPD();J1: iota(V1); isSPD(J1);
V2: randomSPD();J2: iota(V2); isSPD(J2);
dH12: float(HilbertVPMDistance(J1,J2));

alpha:random(1.0);
J12:alpha*J1+(1-alpha)*J2; /* Hilbert pregeodesic */
dH1mid12: float(HilbertVPMDistance(J1,J12));
dHmid122: float(HilbertVPMDistance(J2,J12));
/* check Hilbert VPM line segment is a pregeodesic */
dH12-(dH1mid12+dHmid122);float(%); /* zero */
\end{verbatim}


\bibliographystyle{plain}
\bibliography{FinslerDFS-SPDBib}
\end{document}